%% file: plos_revision_appendices_reordered.tex
\documentclass[10pt,letterpaper]{article}
\usepackage[top=0.85in,left=2.75in,footskip=0.75in]{geometry}

\usepackage{amsmath,amssymb}

\usepackage{changepage}

\usepackage{textcomp,marvosym}

\usepackage[numbers]{natbib}

\usepackage{nameref,hyperref}

\usepackage[right]{lineno}

\usepackage{microtype}
\DisableLigatures[f]{encoding = *, family = * }

\usepackage[table]{xcolor}

\usepackage{array}
\usepackage{booktabs}

\input{math_commands.tex}

\usepackage{tablefootnote}

\usepackage{mathtools}
\usepackage{amsthm}

\usepackage[capitalize,noabbrev]{cleveref}

\theoremstyle{plain}

\newtheorem{claim}{Claim}

\theoremstyle{definition}

\theoremstyle{remark}

\DeclareMathOperator*{\diag}{diag}

\newcolumntype{+}{!{\vrule width 2pt}}

\newlength\savedwidth

\raggedright
\usepackage[aboveskip=1pt,labelfont=bf,labelsep=period,justification=raggedright,singlelinecheck=off]{caption}

\makeatletter
\renewcommand{\@biblabel}[1]{\quad#1.}
\makeatother

\usepackage{lastpage,fancyhdr,graphicx}
\usepackage{epstopdf}
\fancyheadoffset[L]{2.25in}
\fancyfootoffset[L]{2.25in}
\usepackage{etoc}

\begin{document}
\vspace*{0.2in}

\begin{flushleft}
{\Large
\textbf\newline{When predict can also explain: few-shot prediction to select better neural latents} 
}
\newline
\\
Kabir V. Dabholkar\textsuperscript{1}, 
Omri Barak \textsuperscript{2} 
\\
\bigskip
\textbf{1} Faculty of Mathematics,
  Technion -- Israel Institute of Technology, Haifa, Israel
\\
\textbf{2} Rappaport Faculty of Medicine and Network Biology Research Laboratory, Technion - Israel Institute of Technology, Haifa, Israel
\\
\bigskip

%
%






\end{flushleft}
\section*{Abstract}
Latent variable models serve as powerful tools to infer underlying dynamics from observed neural activity. Ideally, the inferred dynamics should align with true ones. However, due to the absence of ground truth data, prediction benchmarks are often employed as proxies. One widely-used method, \emph{co-smoothing}, involves jointly estimating latent variables and predicting observations along held-out channels to assess model performance. In this study, we reveal the limitations of the co-smoothing prediction framework and propose a remedy. Using a student-teacher setup, we demonstrate that models with high co-smoothing can have arbitrary extraneous dynamics in their latent representations. To address this, we introduce a secondary metric — \emph{few-shot co-smoothing}, performing regression from the latent variables to held-out neurons in the data using fewer trials. Our results indicate that among models with near-optimal co-smoothing, those with extraneous dynamics underperform in the few-shot co-smoothing compared to `minimal' models that are devoid of such dynamics. We provide analytical insights into the origin of this phenomenon and further validate our findings on four standard neural datasets using a state-of-the-art method: STNDT. In the absence of ground truth, we suggest a novel measure to validate our approach. By cross-decoding the latent variables of all model pairs with high co-smoothing, we identify models with minimal extraneous dynamics. We find a correlation between few-shot co-smoothing performance and this new measure. In summary, we present a novel prediction metric designed to yield latent variables that more accurately reflect the ground truth, offering a significant improvement for latent dynamics inference.

\section*{Author summary}
The availability of large scale neural recordings encourages the development of methods to fit models to data. How do we know that the fitted models are loyal to the true underlying dynamics of the brain? A common approach is to use prediction scores that use one part of the available data to predict another part. The advantage of predictive scores is that they are general: a wide variety of modelling methods can be evaluated and compared against each other. But does a good predictive score guarantee that we capture the true dynamics in the model?

We investigate this by generating synthetic neural data from one model, fitting another model to it, ensuring a high predictive score, and then checking if the two are similar. The result: only partially. We find that the high scoring models always contain the truth, but may also contain additional `made-up' features. We remedy this issue with a secondary score that tests the model's generalisation to another set of neurons with just a few examples. We demonstrate its applicability with synthetic and real neural data. 

\section*{Introduction}

In neuroscience, we often have access to simultaneously recorded neurons during certain behaviors. These observations, denoted $\mX$, offer a window onto the actual hidden (or latent) dynamics of the relevant brain circuit, denoted $\mZ$ \cite{vyas_computation_2020}. Although, in general, these dynamics can be complex and high-dimensional, capturing them in a concrete mathematical model opens doors to reverse-engineering, revealing simpler explanations and insights \cite{barak_recurrent_2017,sussillo_opening_2013}. Inferring a model of the $\mZ$ variables, $\hat \mZ$, also known as latent variable modeling (LVM), is part of the larger field of system identification with applications in many areas outside of neuroscience, such as fluid dynamics \citep{vinuesa_enhancing_2022} and finance \citep{bauwens_stochastic_2005}.

Because we don't have ground truth for $\mZ$, prediction metrics on held-out parts of $\mX$ are commonly used as a proxy \citep{pei_neural_2021}. However, it has been noted that prediction and explanation are often distinct endeavors \citep{shmueli2010explain}. For instance, \cite{versteeg_expressive_2024} use an example where ground truth is available to show how different models that all achieve good prediction nevertheless have varied latents that can differ from the ground truth. Such behavior might be expected when using highly expressive models with large latent spaces. Bad prediction with good latents is demonstrated by \cite{koppe_identifying_2019} for the case of chaotic dynamics. 

Various regularisation methods on the latents have been suggested to improve the similarity of $\mZ$ to the ground truth, such as recurrence and priors on external inputs \citep{pandarinath_inferring_2018}, low-dimensionality of trajectories \citep{sedler_expressive_2023}, low-rank connectivity \citep{Valente2022neurips,pals_inferring_2024}, injectivity constraints from latent to predictions \citep{versteeg_expressive_2024}, low-tangling \citep{perkins_simple_2023}, and piecewise-linear  dynamics \citep{linderman_bayesian_2017}. However, the field lacks a quantitative, \textit{prediction-based} metric that credits the simplicity of the latent representation—an aspect essential for interpretability and ultimately scientific discovery, while still enabling comparisons across a wide range of LVM architectures.

Here, we characterize the diversity of model latents achieving high \textit{co-smoothing}, a standard prediction-based framework for Neural LVMs, and demonstrate potential pitfalls of this framework (see Methods for a glossary of terms). We propose a few-shot variant of co-smoothing which, when used in conjunction with co-smoothing, differentiates varying latents. We verify this approach both on synthetic data settings and a state-of-the-art method on neural data, providing an analytical explanation of why it works in simple settings.

\section*{Co-smoothing: a cross-validation framework\label{label:formalisation co-smoothing}}

Let $\mX \in \sZ_{\geq 0}^{T \times N}$ be spiking neural activity of $N$ channels recorded over a finite window of time, i.e., a \textit{trial}, and subsequently quantised into $T$ time-bins. $\emX_{t,n}$ represents the number of spikes in channel $n$ during time-bin $t$. The dataset $\mathcal X:=\{\mX^{(i)}\}_{i=1}^S$, partitioned as $\mathcal X^\text{train}$ and $\mathcal X^\text{test}$, consists of $S$ trials of the experiment. The latent-variable model (LVM) approach posits that each time-point in the data $\boldsymbol{x}^{(i)}_{t}$ is a noisy measurement of a latent state $\boldsymbol{z}^{(i)}_{t}$.

To infer the latent trajectory $\mZ$ is to learn a mapping $f:\mX \mapsto \hat \mZ$. On what basis do we validate the inferred $\hat \mZ$? We cannot access the ground truth $\mZ$, so instead we test the ability of $\hat \mZ$ to predict unseen or held-out data. Data may be held-out in time, e.g., predicting future data points from the past, or in space, e.g., predicting neural activities of one set of neurons (or channels) based on those of another set. The latter is called co-smoothing \citep{pei_neural_2021}.

The set of $N$ available channels is partitioned into two: $N^{\text{in}}$ held-in channels and $N^{\text{out}}$ held-out channels. The $S$ trials are partitioned into train and test. During training, both channel partitions are available to the model and during test, only the held-in partition is available. During evaluation, the model must generate the $T\times N^\text{out}$ rate-predictions $R_{:,\text{out}}$ for the held-out partition. This framework is visualised in Fig.~\ref{fig:hmms}A.

Importantly, the encoding-step or inference of the latents is done using a full time-window, i.e., analogous to \textit{smoothing} in control-theoretic literature, whereas the decoding step, mapping the latents to predictions of the data is done on individual time-steps:
\begin{align}
\boldsymbol{\hat z}_t&=f(\mX_{:,\text{in}};t)\label{eq:encoder}\\
\boldsymbol{r}_{t,\text{out}}&=g(\boldsymbol{\hat z}_t)\label{eq:time-step constraint},
\end{align} 
where the subscripts `$\text{in}$' and `$\text{out}$' denote partitions of the neurons (Fig.~\ref{fig:hmms}B). During evaluation, the held-out data from test trials $\mX_{:,\text{out}}$ is compared to the rate-predictions $\mR_{:,\text{out}}$ from the model using the co-smoothing metric $\mathcal Q$ defined as the normalised log-likelihood, given by: 
\begin{align}\label{eq:co-smoothing}
Q(\emR_{t,n},\emX_{t,n}) &:= \frac{1}{\mu_n\log 2}\bigg(\mathcal L(\emR_{t,n};\emX_{t,n}) - \mathcal L(\bar r_n;\emX_{t,n})\bigg)\\
\mathcal Q^{\text{test}} &:= \sum_{n\in\text{held-out}}  \sum_{i\in\text{test}}\sum_{t=1}^TQ(\emR^{(i)}_{t,n},\emX^{(i)}_{t,n}),
\end{align}

where $\mathcal L$ is poisson log-likelihood, $\bar r_n=\frac{1}{TS}\sum_i\sum_t \emX^{(i)}_{t,n}$ is a the mean rate for channel $n$, and $\mu_n:=\sum_i\sum_t \emX^{(i)}_{t,n}$ is the total number of spikes, following \cite{pei_neural_2021}. 

Thus, the inference of LVM parameters is performed through the optimization:
\begin{align}
f^*,g^* = \text{argmax}_{f,g} \mathcal Q^\text{train} \label{eq: co-smoothing optimization}
\end{align}

using $\mathcal X^\text{train}$, without access to the test trials from $\mathcal X^\text{test}$. For clarity, apart from \eqref{eq: co-smoothing optimization}, we report only $\mathcal Q^\text{test}$, omitting the superscript.

\begin{figure*}[ht!]
    \centering
    \includegraphics[width=\textwidth]{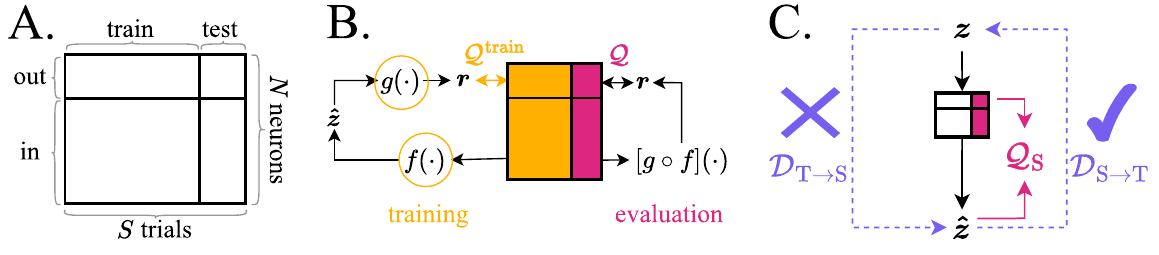}

    \caption{Prediction framework and its relation to ground truth. \textbf{A.} To evaluate a neural LVM with co-smoothing, the dataset is partitioned along the neurons and trials axes. \textbf{B.} The held-in neurons are used to infer latents $\hat z$, while the held-out serve as targets for evaluation. The encoder $f$ and decoder $g$ are trained jointly to maximise co-smoothing $\mathcal Q$. After training, the composite mapping $g\circ f$ is evaluated on the test set. \textbf{C.} We hypothesise that models with high co-smoothing may have an asymmetric relationship to the true system, ensuring that model representation contains the ground truth, but not vice-versa. We reveal this in a synthetic student(S)-teacher(T) setting by the unequal performance of regression on the states in the two directions. $\mathcal D_{u\rightarrow v}$ denote decoding error of model $v$ latents $\boldsymbol{z}_v$ from model $u$ latents $\boldsymbol{z}_u$.}
    \label{fig:hmms}
\end{figure*}

\section*{Good co-smoothing does not guarantee correct latents}\label{sec:does not guarantee}

It is common to assume that being able to predict held-out parts of $\mX$ will guarantee that the inferred latent aligns with the true one \citep{macke2011,pei_neural_2021,wu2018learning,meghanath2023inferring,keshtkaran_large-scale_2022,keeley2020,le2022stndt,pmlr-v115-she20a,wu2017discovery,zhao2017variational,schimel2022ilqrvae,mullen2024learning,gokcen_disentangling_2022,yu_gaussian-process_2008,perkins_simple_2023}. To test this assumption, we use a student-teacher scenario where we know the ground truth. To compare how two models ($u,v$) align, we infer the latents of both from $\mathcal X^\text{test}$, then do a regression from latents of $u$ to $v$. The regression error is denoted $\mathcal D_{u\rightarrow v}$ (i.e. $\mathcal D_{\text{T}\rightarrow \text{S}}$ for teacher to student decoding).  Contrary to the above assumption, we hypothesize that good prediction guarantees that the true latents are contained within the inferred ones (low $\mathcal D_{\text{S}\rightarrow \text{T}}$), but not vice versa (Fig.~\ref{fig:hmms}C). It is possible that the inferred latents possess additional features, unexplained by the true latents (high $\mathcal D_{\text{T}\rightarrow \text{S}}$).

We demonstrate this phenomenon in three different student-teacher scenarios: task-trained RNNs, Hidden Markov Models (HMMs) and linear gaussian state space models. We start with RNNs, as they are a standard tool to investigate computation through dynamics in neuroscience \cite{versteeg_computation-through-dynamics_2025}, and expand upon the other models in the appendix. A $128$-unit RNN teacher (Methods) is trained on a 2-bit flip-flop task, inspired by working memory experiments. The network receives input pulses and has to maintain the identity of the last pulse (see Methods). The student is a sequential autoencoder, where the encoder $f$ is composed of a neural network that converts observations into an initial latent state, and another recurrent neural network that advances the latent state dynamics \cite{versteeg_computation-through-dynamics_2025} (see Methods).

We generated a dataset of observations from this teacher, and then trained $30$ students with latent-dimensionality $3-64$ on the same teacher data using gradient-based methods (see Methods). Co-smoothing scores of students increased with the size of the latents, but are high for models in the range of 5-15 dimensional latents (\nameref{sec:co-smoothing vs latent size RNNs}). Consistent with our hypothesis, the ability to decode the teacher from the student was highly correlated to the co-smoothing score (Fig.~\ref{fig:hmm graphs} top left). In contrast, the ability to decode the student from the teacher has a very different pattern. For students with low co-smoothing, this decoding is good -- but meaningless. For students with high co-smoothing, there is a large variability, and little correlation to the co-smoothing score (Fig.~\ref{fig:hmm graphs} top right). In this simple example, it would seem that one only needs to increase the dimensionality of the latent until co-smoothing saturates. This minimal value would satisfy both demands. This is not the case for real data, as will be shown below.

What is it about a student model, that produces good co-smoothing with the wrong latents? It's easiest to see this in a setting with discrete latents, so we first show the HMM teacher and two exemplar students -- named ``Good'' and ``Bad'' (marked by green and red arrows in \nameref{sec: does not guarantee HMM}AB) -- and visualise their states and transitions using graphs in Fig.~\ref{fig:hmm graphs}. The teacher is a cycle of 4 steps. The good student contains such a cycle (orange), and the initial distribution is restricted to that cycle, rendering the other states irrelevant. In contrast, the \textit{bad} student also contains this cycle (orange), but the initial distribution is not consistent with the cycle, leading to an extraneous branch converging to the cycle, as well as a departure from the main cycle (both components in dark colour). Note that this does not interfere with co-smoothing, because the emission probabilities of the extra states are consistent with true states, i.e., the emission matrix conceals the extraneous dynamics. In the RNN, we see a qualitatively similar picture, with the bad students having dynamics in task-irrelevant dimensions (Fig.~\ref{fig:hmm graphs} ``Bad'' S).

\begin{figure}
    \centering
    \includegraphics[width=\linewidth]{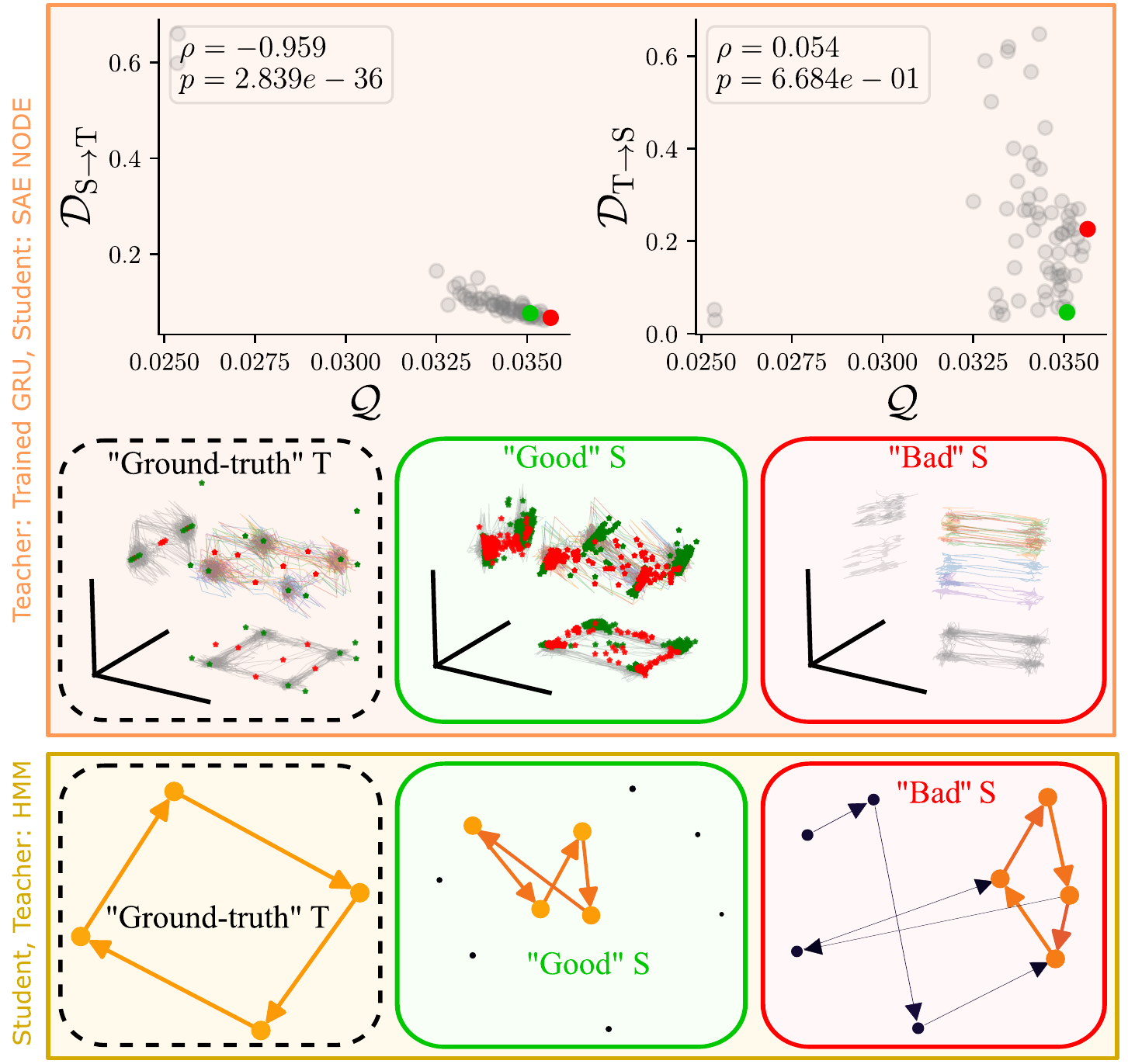}
    \caption{\textbf{Upper panel} Several students, sequential autoencoders (SAE, see Methods), are trained on a dataset generated by a single teacher, a noisy GRU RNN trained on a 2-bit flip flop (2BFF, see Methods). The Student$\rightarrow$Teacher decoding error $\mathcal D_{\text{S}\rightarrow \text{T}}$ is low and tightly related to the co-smoothing score. The Teacher$\rightarrow$Student decoding error $\mathcal D_{\text{T}\rightarrow \text{S}}$ is more varied and uncorrelated to co-smoothing. A score of $\mathcal Q=0$ corresponds to predicting the mean firing-rate for each neuron at all trials and time points. Green and red points are representative "Good" and "Bad" students respectively, whose latents are visualised below along-side the ground truth $\text{T}$. The visualisations are projections of the latents along the top three principal components of the data. The ground truth latents are characterised by $4$ stable states capturing the $2^2$ memory values. This structure is captured in the "Good" student. The bad student also includes this structure in addition to an extraneous variability along the third component.
    \textbf{Lower panel} The same experiment conducted with HMMs. The teacher is a nearly deterministic 4-cycle and students are fit to its noisy emissions. Dynamics in selected models are visualised. Circles represent states, and arrows represent transitions. Circle area and edge thickness reflect fraction of visitations or volume of traffic after sampling the HMM over several trials. The colours also reflect the same quantity -- brighter for higher traffic. 
    Edges with values below $0.01$ are removed for clarity (\nameref{sec:HMMalledges}). The teacher ($M=4$) is a $4$-cycle. Note the prominent 4-cycles (orange) present in the good student ($M=10$), and the bad student ($M=8$). In the good student, the extra states are seldom visited, whereas in the bad student there is significant extraneous dynamics involving these states (dark arrows).
    }
    \label{fig:hmm graphs}
\end{figure}
\section*{Few-shot prediction selects better models}\label{sec:fewshot selects}

Because our objective is to obtain latent models that are close to the ground truth, the co-smoothing prediction scores described above are not satisfactory. Can we devise a new prediction score that will be correlated with ground truth similarity? The advantage of prediction benchmarks is that they can be optimized, and serve as a common language for the community as a whole to produce better algorithms \citep{deng_imagenet_2009}.

We suggest \textbf{few-shot co-smoothing} as a complementary prediction score to co-smoothing, to be used on models with good scores on the latter. Similarly to standard co-smoothing, the functions $g$ and $f$ are trained using all trials of the training data (Fig.~\ref{fig:fewshot framework}A). The key difference is that a separate group of $N^\text{$k$-out}$ neurons is set aside (Table \ref{tab:data dimensions}), and only $k$ trials of these neurons are used to estimate a mapping $g':\hat\mZ_{t,:}\mapsto \mR_{t,\text{$k$-out}}$ (Fig.~\ref{fig:fewshot framework}B), similar to $g$ in \eqref{eq:time-step constraint}. The neural LVM $(f,g,g')$ is then evaluated on both the standard co-smoothing $\mathcal Q$ using $g\circ f$ and the few-shot version $\mathcal Q^k$ using $g' \circ f$ (Fig.~\ref{fig:fewshot framework}C).

\begin{figure}
    \centering
    \includegraphics[width=0.7\linewidth]{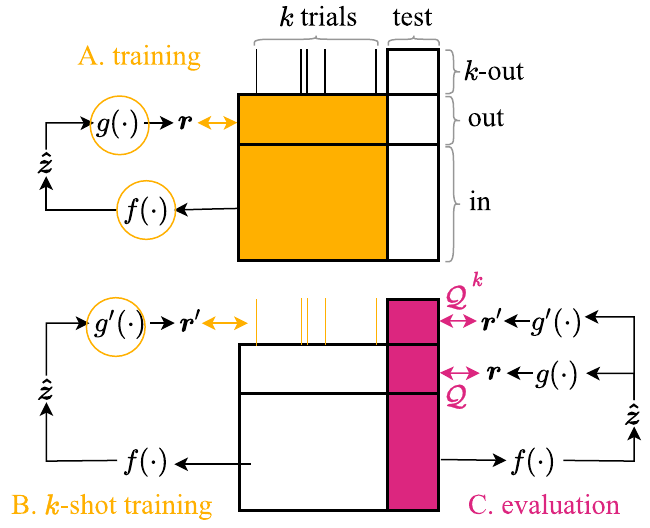}
    \caption{Co-smoothing and few-shot co-smoothing; a composite evaluation framework for Neural LVMs. \textbf{A.} The encoder $f$ and decoder $g$ are trained jointly using held-in and held-out neurons. \textbf{B.} A separate decoder $g'$ is trained to readout $k$-out neurons using only $k$ trials. Meanwhile, $f$ and $g$ are frozen. \textbf{C.} The neural LVM is evaluated on the test set resulting in two scores: co-smoothing $\mathcal Q$ and $k$-shot co-smoothing $\mathcal Q^k$.}
    \label{fig:fewshot framework}
\end{figure}

For small values of $k$, the $\mathcal Q^k$ scores can be highly variable. To reduce this variability, we repeat the procedure $s$ times on independently resampled sets of $k$ trials, producing $s$ estimates of $g'$, each with its own score $\mathcal Q^k$. For each student $\text{S}$, we then report the average score $\langle \mathcal Q^k_\text{S}\rangle$ across the $s$ resamples. A theoretical analysis of the choice of $k$ is given in the next section, with practical guidelines provided in \nameref{sup:how to choose k}. The number of resamples $s$ is chosen empirically to ensure high confidence in the estimated average (Methods).






To demonstrate the utility of the proposed prediction score, we return to the RNN students from Fig.~\ref{fig:hmm graphs} and evaluate $\langle\mathcal Q^k_\text{S}\rangle$ for each. This score provides complementary information about the models, as it is uncorrelated with standard co-smoothing (Fig.~\ref{fig:hmm fewshot}A), and it is not merely a stricter version of co-smoothing (\nameref{sec: not simply hard o-smoothing}). Since we are only interested in models with good co-smoothing, we restrict attention to students satisfying $\mathcal Q_{\text{S}} > \mathcal Q_{\text{T}} - 10^{-3}$. Among these students, despite their nearly identical co-smoothing scores, the $k$-shot scores $\langle\mathcal Q^k_\text{S}\rangle$ are strongly correlated with the ground-truth measure $\mathcal D_{\text{T}\rightarrow \text{S}}$ (Fig.~\ref{fig:hmm fewshot}B). Together, these findings suggest that simultaneously maximizing $\mathcal Q_\text{S}$ and $\langle \mathcal Q^k_\text{S}\rangle$—both prediction-based objectives—produces models with low $\mathcal D_{\text{S}\rightarrow \text{T}}$ and $\mathcal D_{\text{T}\rightarrow \text{S}}$, yielding a more complete measure of model similarity to the ground truth.

\begin{figure}[ht!]
    \centering
    \includegraphics[width=\linewidth]{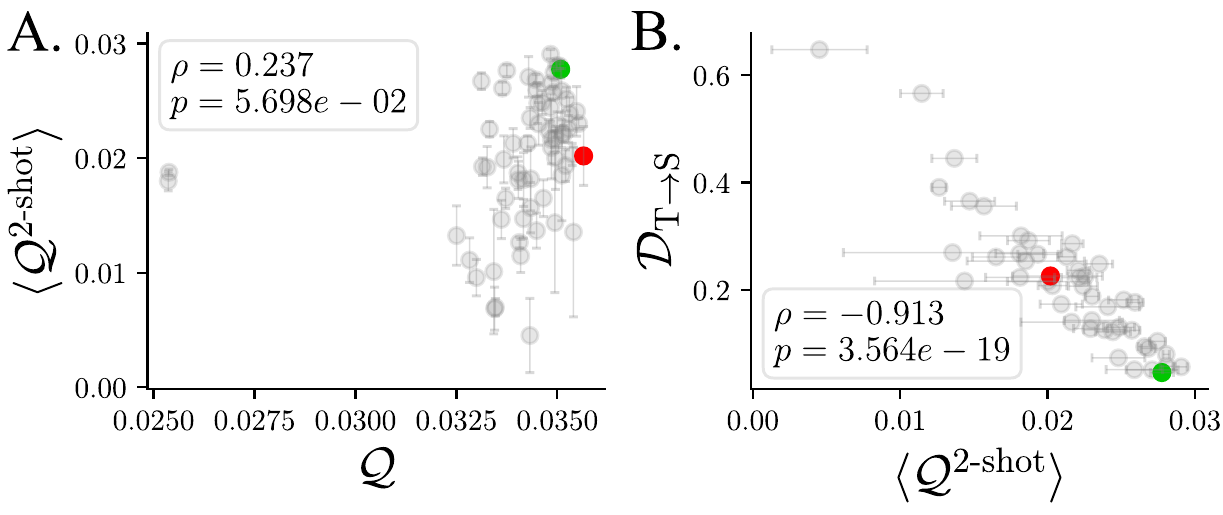}
    \caption{Few-shot prediction selects better models. \textbf{A.} Few-shot measures something new. Student models with high co-smoothing have highly variable $2$-shot co-smoothing, which is uncorrelated to co-smoothing. Error bars reflect standard error of the mean across several few-shot regressions (see Methods). \textbf{B.} For the set of students with high co-smoothing, i.e., satisfying $\mathcal Q>0.034$, $2$-shot co-smoothing to held-out neurons is negatively correlated with decoding error from teacher-to-student. Green and red points represent the example "Good" and "Bad" models (Fig.~\ref{fig:hmm graphs}).}
    \label{fig:hmm fewshot}
\end{figure}

\section*{Why does few-shot work?}

The example HMM and RNN students of Fig.~\ref{fig:hmm graphs} can help us understand why few-shot prediction identifies good models. The students differ in that the \textit{bad} student has more than one state corresponding to the same teacher state. Because these states provide the same output, this feature does not hurt co-smoothing. In the few-shot setting, however, the output of all states needs to be estimated using a limited amount of data. Thus the information from the same amount of observations has to be distributed across more states. We make this data efficiency argument more precise in three settings: linear regression, HMMs, and prototype learning.

The teacher latent is a scalar random variable $z$ and the student latent $\boldsymbol{\hat z}$ is a random $p$-vector, whose first coordinate is $z$ and the remaining $p-1$ coordinates are the extraneous noise:
\begin{align}
    \boldsymbol{\hat z} := \begin{bmatrix}
        z \quad 
        \underbrace{\xi_1 \quad \xi_2 \quad \dots \quad \xi_{p-1}}_{\text{extraneous noise}}
    \end{bmatrix}^T,
\end{align}
where $\xi_j\sim\mathcal N(0,\sigma_\text{ext}^2)$. In other words, a single teacher state is represented by several possible student states.

Next, we model the neural-data -- noisy observations of the teacher latent $x:=z+\epsilon$, where $\epsilon\sim\mathcal (0,\sigma_\text{obs}^2)$. The few-shot learning is captured by minimum-norm $k$-shot least-squares linear regression (LR):
\begin{align}
    \boldsymbol{\hat w}:= \argmin_w \bigg\{ \Vert \boldsymbol{w} \Vert^2 : \boldsymbol{w} \text{ minimises } \sum_{i=1}^k\Vert x^{(i)}-\boldsymbol{w}^T\boldsymbol{\hat z}^{(i)}\Vert^2\bigg\},
\end{align}
where $\Vert \cdot \Vert$ is the $2$-norm.

The generalisation error of the few-shot learner is given by:
\begin{align}
    \mathcal R^k  = \big\langle (\boldsymbol{\hat z}^T\boldsymbol{w}^* - \boldsymbol{\hat z}^T\boldsymbol{\hat w})^2 \big\rangle_{z,\xi_1,\dots,\xi_p,\epsilon},
\end{align}
where $\boldsymbol{w}^*=\begin{bmatrix}1 & 0 & \dots & 0\end{bmatrix}^T$ is the true mapping.

We solve for $\langle \mathcal R^k \rangle$ as $k,p\to\infty,p/k\to\gamma\in(0,\infty)$ using the theory of \citet{hastie_surprises_2022}, and demonstrate a good fit to numerical simulations at finite $p,k$ (Methods). We do similar analyses for Bernoulli HMM latents with maximum likelihood estimation of the emission parameters (Methods) and binary classification with prototype learning (BCPL) \cite{sorscher_neural_2022} (Methods). 

Across the three scenarios, model performance decreases with extraneous variability (Fig.~\ref{fig:kshot_theory}). Crucially, this difference appears at small $k$, and vanishes as $k\to\infty$. With HMMs and BCPL this is a gradual decrease, while in LR, there is a known critical transition at $p=k$ \cite{hastie_surprises_2022,belkin_reconciling_2019,nakkiran_deep_2021}.

Interestingly, the scenarios differ in the bias-variance decomposition of their performance deficits. In LR, extraneous noise leads to increased bias (identical variance), whereas in the HMM and BCPL, it leads to increased variance (zero bias).

How does one choose the value of $k$ in practice? The intuition and theoretical results suggest that we want the smallest possible value. In real data, however, we expect many sources of noise that could make small values impractical. For instance, for low firing rates, small $k$ values can mean that some neurons will not have any spikes in $k$ trials and thus there will nothing to regress from. Our suggestion is therefore to use the smallest value of $k$ that allows robust estimation of few-shot co-smoothing. (\nameref{sup:how to choose k}) shows the effect of this choice for various datasets.


\begin{figure}
    \centering
    \includegraphics[width=\linewidth]{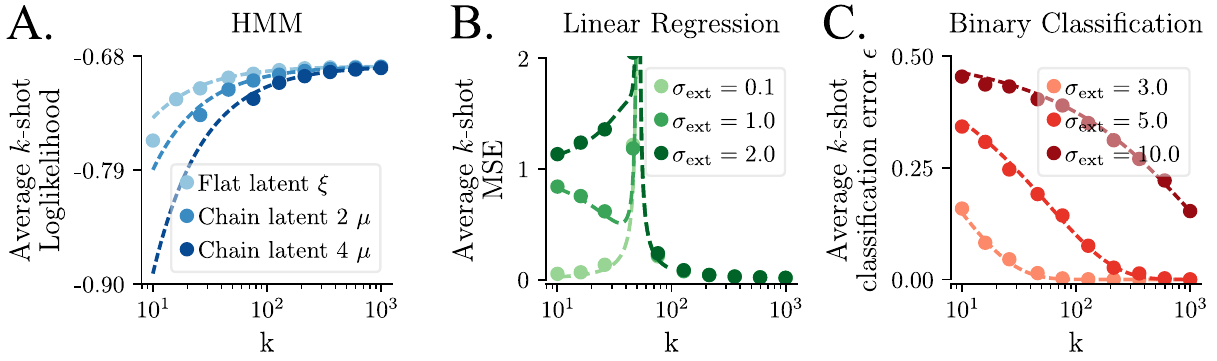}
    \caption{Theoretical analysis of $k$-shot learner performance as a function of $k$ and extraneous noise $\sigma_\text{ext}$, in three different settings. Points show numerical simulations and dashed lines show analytical theory. \textbf{A.} Hidden Markov Models (HMMs) (Methods), Bernoulli observations, MLE estimator. \textbf{B.} Minimum norm least squares linear regression with $\sigma_\text{obs}=0.3$ and $p=50$ (main text and Methods). \textbf{C.} binary classification, prototype learning (Methods).}
    \label{fig:kshot_theory}
\end{figure}

\section*{SOTA LVMs on neural data}\label{sec: SOTA}

In previous sections, we showed that models with near perfect co-smoothing may possess latents with extraneous dynamics. We established this in a synthetic student-teacher setting with RNNs, HMMs and LGSSM models. 

To show the applicability in more realistic scenarios, we consider four datasets \texttt{mc\_maze\_20} \cite{churchland_cortical_2010}, \texttt{mc\_rtt\_20} \cite{odoherty_nonhuman_2018}, \texttt{dmfc\_rsg\_20} \cite{sohn_bayesian_2019}, \texttt{area2\_bump\_20} \cite{chowdhury_area_2020} from the Neural Latent Benchmarks suite \cite{pei_neural_2021} (see Methods). They consist of neural activity (spikes) recorded from various cortical regions of monkeys as they perform specific tasks. The \texttt{20} indicates that spikes were binned into $20ms$ time bins. We trained several SpatioTemporal Neural Data Transformers (STNDTs) \citep{le_stndt_2022,ye_representation_2021,nguyen_transformers_2019,huang_improving_2020}, that achieve near state-of-the-art (SOTA) co-smoothing on these datasets. We evaluate co-smoothing on a test set of trials and define the set of models with the best co-smoothing (see Methods and Table \ref{tab:data dimensions}).

A key component of training modern neural network architectures such as STNDT is the random sweep of hyperparameters, a natural step in identifying an optimal model for a specific data set \cite{keshtkaran_large-scale_2022}. This process generates several candidate solutions to the optimization problem \eqref{eq: co-smoothing optimization}, yielding models with similar co-smoothing scores but, as we demonstrate in this section, varying amounts of extraneous dynamics.

\noindent
\textbf{Two proxies for $\mathcal D_{\text{T}\rightarrow\text{S}}$: cycle consistency and cross-decoding.}

To reveal extraneous dynamics in the synthetic examples (RNNs, HMMs), we had access to ground truth that enabled us to directly compare the student latent to that of the teacher. With real neural data, we do not have this privilege. This limitation has been recognized in the past and a proxy was suggested \cite{versteeg_expressive_2024,versteeg_computation-through-dynamics_2025,zhu_unpaired_2017} -- \textit{cycle consistency}. Instead of decoding the student latent from the teacher latent, cycle consistency attempts to decode the student latent $\boldsymbol{\hat z}$ from the student's own \textit{rate prediction} $\boldsymbol{r}$. In our notation this is  $\mathcal D_{\boldsymbol{r}\rightarrow \boldsymbol{\hat z}}$  (Fig.~\ref{fig:cross decoding}A and Methods). If the student has perfect co-smoothing, this should be equivalent to  $\mathcal D_{\text{T}\rightarrow\text{S}}$  as it would ensure that teacher and student have the same rate-predictions $\boldsymbol{r}$. 

Because we cannot rely on perfect co-smoothing, we also suggest a novel metric -- \emph{cross-decoding} -- where we compare the models to each other. The key idea is that all high co-smoothing models contain the teacher latent. One can then imagine that each student contains a selection of several extraneous features. The best student is the one containing the least such features, which would imply that all other students can decode its latents, while it cannot decode theirs (Fig.~\ref{fig:cross decoding}B). Instead of computing $\mathcal D_{\text{S}\rightarrow\text{T}}$ and $\mathcal D_{\text{T}\rightarrow\text{S}}$ as in Fig.~\ref{fig:hmm graphs}, we perform decoding from latents of model $u$ to model $v$ ($\mathcal D_{u\rightarrow v}$) for every pair of models $u$ and $v$ using linear regression and evaluating an $R^2$ score for each mapping (see Methods). In Fig.~\ref{fig:cross decoding}C the results are visualised by a $U \times U$ matrix with entries $\mathcal D_{u\rightarrow v}$ for all pairs of models $u$ and $v$.
The ideal model $v^*$ would have no extraneous dynamics, therefore, all the other models should be able to decode its latents perfectly, i.e., $\mathcal D_{u\rightarrow v^*}=0 \ \forall  \ u$. Provided a large and diverse population of models only the `pure' ground truth would satisfy this condition. To evaluate how close a model $v$ is to the ideal $v^*$ we propose a simple metric: the column average $\langle \mathcal D_{u\rightarrow v}\rangle_u$. This will serve as proxy for the distance to ground truth, analogous to $ \mathcal D_{\text{T}\rightarrow \text{S}}$ in Fig.~\ref{fig:hmm fewshot}. We validate this procedure using the RNN student-teacher setting in Fig.~\ref{fig:cross decoding}D, where we show that $\langle \mathcal D_{u\rightarrow v}\rangle_u$ is highly correlated to the ground truth measure $\mathcal D_{\text{T}\rightarrow\text{S}}$. We also validate cycle-consistency $\mathcal D_{\boldsymbol{r}\rightarrow \boldsymbol{\hat z}}$ against $\mathcal D_{\text{T}\rightarrow\text{S}}$ using the RNN setting (Fig.~\ref{fig:cross decoding}E). In both cases we find a high correlation between the metrics.

\begin{figure}[ht!]
    \centering
    \includegraphics[width=\linewidth]{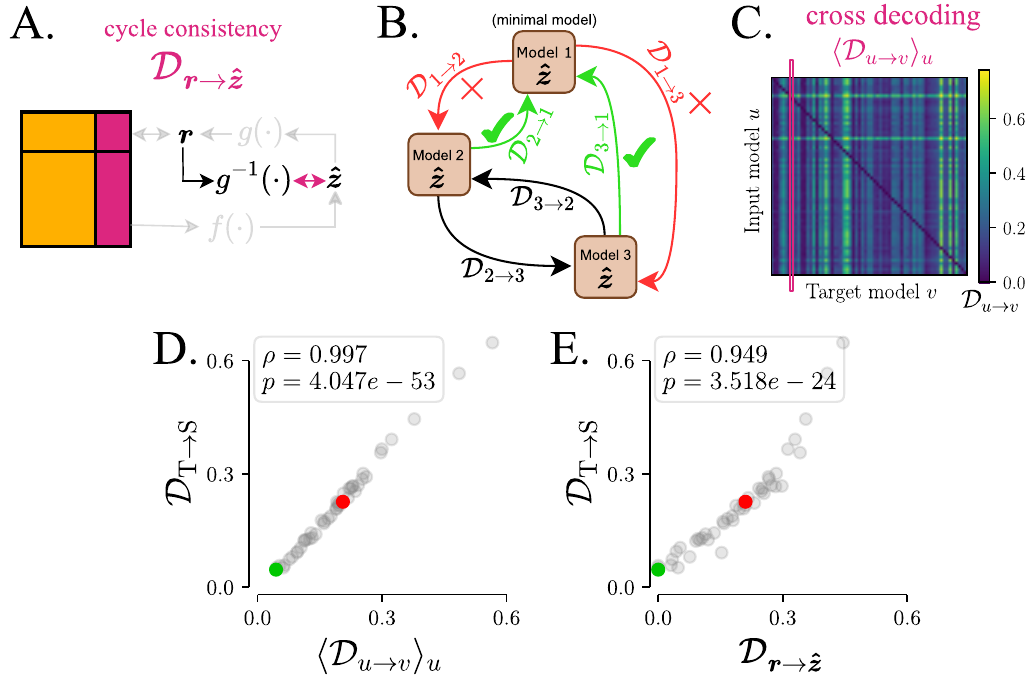}
    \caption{Cycle consistency and cross-decoding as a proxy for distance to the ground truth in the absence of ground-truth. \textbf{A} \textit{Cycle consistency} $\mathcal D_{\boldsymbol r\rightarrow \boldsymbol {\hat z}}$ \cite{zhu_unpaired_2017,versteeg_expressive_2024,versteeg_computation-through-dynamics_2025} involves learning a mapping $g^{-1}$ from the rates $\boldsymbol{r}$ back to the latents $\boldsymbol{\hat z}$ (see Methods). \textbf{B} The latents of each pair of models are \textit{cross-decoded} from one another.  Minimal models can be fully decoded by all models but extraneous models only by some. \textbf{C} Cross-decoding matrix for SAE NODE models trained on data from the NoisyGRU (Fig.~\ref{fig:hmm graphs}). \textbf{D, E} For models with high co-smoothing ($\mathcal Q>0.035$) the proxy metrics -- cross-decoding column average $\langle \mathcal D_{u\rightarrow v}\rangle_u$, and cycle-consistency $\mathcal D_{\boldsymbol{r}\rightarrow\boldsymbol{z}}$) -- are both highly correlated to ground truth $\mathcal D_{\text{T}\rightarrow\text{S}}$.}
    \label{fig:cross decoding}
\end{figure}

\begin{figure}[ht!]
    \centering
    \includegraphics[width=\textwidth]{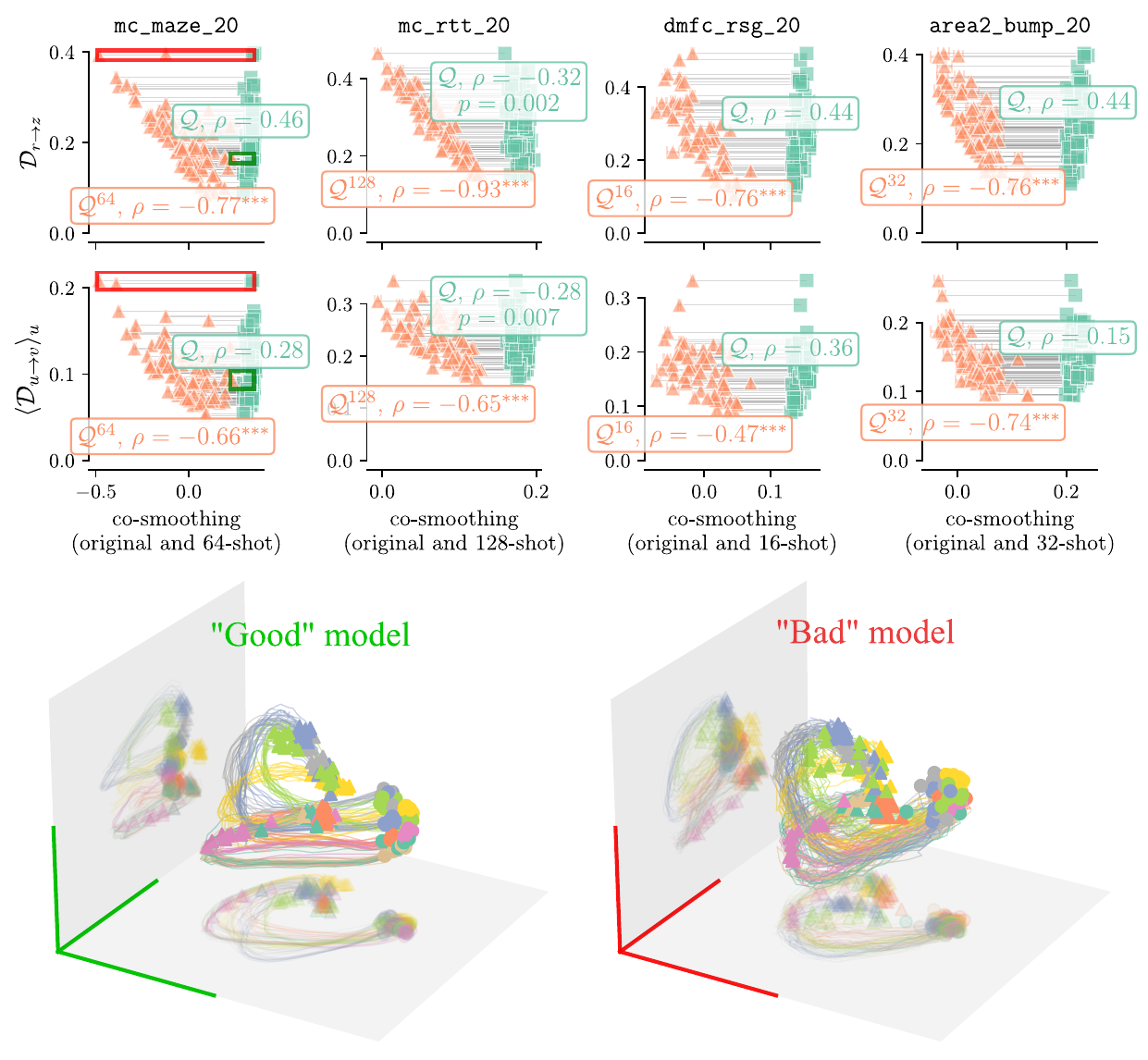}
    \caption{Few-shot scores $\langle\mathcal Q^{k\text{-shot}}\rangle$ correlate with the proxies of distance to the ground truth, cycle-consistency $\mathcal D_{\boldsymbol{r}\rightarrow\boldsymbol{z}}$ and the cross-decoding column average $\langle \mathcal D_{u\rightarrow v}\rangle_u$. We train several STNDT models on four neural recordings from monkeys  \cite{churchland_cortical_2010,odoherty_nonhuman_2018,sohn_bayesian_2019,chowdhury_area_2020}, curated by \citet{pei_neural_2021} and filter for models with high co-smoothing $\mathcal Q>0.8 \times \max(\mathcal Q)$. The few-shot co-smoothing scores $\langle\mathcal Q^{k-\text{shot}}\rangle$ negatively correlate with the two proxies $\mathcal D_{\boldsymbol{r}\rightarrow\boldsymbol{z}}$ and $\langle \mathcal D_{u\rightarrow v}\rangle_u$  (orange points), while regular co-smoothing $\mathcal Q$ (turquoise points) does not (one-tailed p-values shown for $p<0.05$ and ${}^{***}$ for $p<0.001$). Green and red arrows indicate the extreme models whose latents are visualised below. $\mathcal Q$ values may be compared against an \href{https://eval.ai/web/challenges/challenge-page/1256/leaderboard/3183}{EvalAI leaderboard} \citep{pei_neural_2021}. Note that we evaluate using an offline train-test split, not the true test set used for the leaderboard scores, for which held-out neuron data is not publicly accessible. (Bottom) Principal component analysis of the latent trajectories of two STNDT models trained on \texttt{mc\_maze\_20} with similar co-smoothing scores but contrasting few-shot co-smoothing. The ``Good'' model scores $\mathcal Q=0.341$, $\langle\mathcal Q^{64\text{-shot}}\rangle=0.292$ and the ``Bad'' model $\mathcal Q=0.342$, $\langle\mathcal Q^{64\text{-shot}}\rangle=0.012$. The trajectories are coloured by task conditions and start at a circle and end in a triangle.}
    \label{fig:SOTA_colsums}
\end{figure}

Having developed a proxy for the ground truth we can now correlate it with the few-shot co-smoothing $\langle\mathcal Q^{k\text{-shot}}\rangle$ to held-out neurons. Following the disucssion in the previous section, we choose the smallest value of $k$ that ensures no trials with zero spikes (\nameref{sup:how to choose k}). Fig.~\ref{fig:SOTA_colsums} shows a negative correlation of $\langle Q^{k\text{-shot}}\rangle$ with both proxy measures $\mathcal D_{\boldsymbol{r}\rightarrow \boldsymbol{\hat z}}$ and $\langle \mathcal D_{u\rightarrow v}\rangle_u$ across the STNDT models in the four data sets. Moreover, regular co-smoothing $\mathcal Q$ for the same models is relatively uncorrelated with these measures. As an illustration of the latents of different models, Fig.~\ref{fig:SOTA_colsums}(bottom) shows the PCA projection of latents from two STNDT models trained on \texttt{mc\_maze\_20}. Both have high co-smoothing scores but differ in their few-shot scores $\langle \mathcal Q^{k\text{-shot}}\rangle$. We note smoother trajectories and better clustering of conditions in the model with higher $\langle\mathcal Q^{k\text{-shot}}\rangle $. 


\section*{Discussion}\label{sec:discussion}

Latent variable models (LVMs) aim to infer the underlying latents using observations of a target system. We showed that co-smoothing, a common prediction measure of the goodness of such models, cannot discriminate between LVMs containing only the true latents and those with additional extraneous dynamics.

We propose a complementary prediction measure: few-shot co-smoothing. After training the encoder that translates data observations to latents, we use only a few ($k$) trials to train a new decoder. Using several synthetic datasets generated from trained RNNs and two other state-space architectures, we show numerically and analytically that this measure correlates with the distance of model latents to the ground truth. 

We demonstrate the applicability of this measure to four datasets of monkey neural recordings with a transformer architecture \cite{le_stndt_2022,ye_representation_2021} that achieves near state-of-the-art (SOTA) results on all datasets. This required developing a new proxy to ground truth -- cross-decoding. For each pair of models, we try to decode the latents of one from the latents of the other. Models with extraneous dynamics showed up as poor target latents on average, and vice versa. 

Our work is related to a recent study that addresses benchmarking LVMs for neural data by developing benchmarks and metrics using only synthetic data  - Computation through dynamics benchmark \cite{versteeg_computation-through-dynamics_2025}. This study similarly tackles the issue of extraneous dynamics, primarily using ground-truth comparisons and cycle consistency. Our cross-decoding metric complements cycle consistency\cite{versteeg_expressive_2024,versteeg_computation-through-dynamics_2025} as a proxy for ground truth. Cycle consistency has the advantage that it is defined on single models, compared with cross-decoding that depends on the specific population of models used. Cycle consistency has the disadvantage that it relies on the rate predictions being perfect proxies to the true dynamics. In the datasets we analyzed here, both measures provided very similar results.  An interesting extension would be to use the cross-decoding metric as another method to select good models. However, its computational cost is high, as it requires training a population of models and comparing them pairwise. Additionally, it is less universal and standardized than few-shot co-smoothing, as it depends on a specific 'jury' of models.

Several works address the issue of extraneous dynamics through regularization of dimensionality, picking the minimal dimensional or rank-constrained model that still fits the data \cite{versteeg_expressive_2024,sedler_expressive_2023,Valente2022neurips,pals_inferring_2024}. Usually, these constraints are accompanied by poorer co-smoothing scores compared to their unconstrained competitors, and the simplicity of these constrained models often goes uncredited by standard prediction-based metrics. Classical measures like AIC \cite{parzen_information_1998} and BIC \cite{schwarz_estimating_1978} address the issue of overfitting by penalising the number of parameters, but are less applicable given the success of overparameterised models \cite{belkin_reconciling_2019}.
We believe these approaches may not scale well to increasingly larger datasets \cite{altan_estimating_2021}, noting studies reporting that neural activity is not finite-dimensional but exhibits a scale-free distribution of variance \cite{stringer_high-dimensional_2019,stringer_spontaneous_2019}. Our few-shot co-smoothing metric, by contrast, does not impose dimensional constraints and instead leverages predictive performance on limited data to identify models closer to the true latent dynamics, potentially offering better scalability for complex, large-scale neural datasets.

While the combination of student-teacher and SOTA results presents a compelling argument, we address a few limitations of our work. Regarding few-shot regression, while the Bernoulli HMM scenario has a closed-form solution (the maximum likelihood estimate), the Poisson GLM regression for SOTA models is optimized iteratively and is sensitive to the L2 hyperparameter $\alpha$. In our results, we select $k$ and $\alpha$ that distinguish models in our candidate model sets, yielding moderate/high few-shot scores for some models and low scores for others. This is an empirical choice that must be made for each dataset and model set. The few-shot training of $g'$ is computationally inexpensive and can thus be evaluated over a range of values to find the optimal ones.

Overall, our work advances latent dynamics inference in general and prediction frameworks in particular. By exposing a failure mode of standard prediction metrics, we guide the design of inference algorithms that account for this issue. Furthermore, the few-shot co-smoothing metric can be incorporated into existing benchmarks, helping the community build models that are closer to the desired goal of uncovering latent dynamics in the brain.

\section*{Methods}
\subsection*{Glossary}


\begin{description}
  \item[\textbf{Latent variable model (LVM)} ($f$ and $g$)]: A function mapping neural time-series data to an inferred latent space ($f$). The latents can then be used to predict held-out data ($g$). 
  \item[\textbf{Smoothing}]: mapping a sequence of observations $\mX_{1:T}$ to a sequence of inferred latents $\hat \mZ_{1:T}$. It is often formalised as a conditional probability $p(\hat \mZ_{1:T}|\mX_{1:T})$.
  \item[\textbf{Extraneous dynamics}]: the notion that inferred latent variables may contain features and temporal structure not present in the true system from which the data was observed.
  \item[\textbf{Co-smoothing} ($\mathcal Q$)]: A metric evaluating LVMs by their ability to predict the activity of held-out neurons $\mX_{1:T,\text{out}}$ provided held-in neural activity $\mX_{1:T,\text{in}}$ over a window of time. The two sets of neurons are typically random subsets from a single population.
  \item[\textbf{Few-shot co-smoothing} ($\mathcal Q^{k\text{-shot}}$)]: A variant of co-smoothing in which the mapping from latents to held-out neurons ($g'$) is learned from a small number of trials. 
  \item[\textbf{State-of-the-art (SOTA)}]: the best performing method or model current available in the field. This is usually based on a specific benchmark, i.e., a dataset and associated evaluation metric. In active fields the SOTA is constantly improving.
  \item[\textbf{Cycle consistency ($\mathcal D_{\boldsymbol r\rightarrow \boldsymbol {\hat z}}$)}]: a measure of \textit{extraneousness} of model latents as compared to their rate predictions. Computed by learning and evaluating the inverse mapping from rate predictions to latents.
  \item[\textbf{Cross-decoding ($\mathcal D_{u\rightarrow v}$)}]: another measure of model \textit{extraneousness}. It is evaluated on a population of models trained on the same dataset. It involves regressing from one model latents to another model, for all pairs in the population. A scalar measure is the obtained for each model: the cross-decoding column mean $\langle \mathcal D_{u\rightarrow v}\rangle_u$. It reflects the average 'decodability' of a model, by all the other models.
\end{description}

\subsection*{Student-teacher Recurrent Neural Networks (RNN)}

Both teacher and student are based on an adapted version of \citet{versteeg_computation-through-dynamics_2025}. In the following, we provide a brief description. 

\noindent
\textbf{Teacher}

We train a noisy 64 Gated Recurrent Unit (NoisyGRU) RNN \cite{chung_empirical_2014}, on a 2-bit flip flop 2BFF task \cite{sussillo_opening_2013}, implemented by \citet{versteeg_computation-through-dynamics_2025}. The GRU RNN follows standard dynamics, which we repeat here using the typical notation of GRUs. This notation is not consistent with the Results section, and we explain the relation below.
\begin{align}
\mathbf{h}_0 &= \mu + \eta; \quad \eta \sim \mathcal{N}(0, 0.05) \\
\mathbf{z}_t &= \sigma(\mathbf{W}_z \mathbf{x}_t + \mathbf{U}_z \mathbf{h}_{t-1} + \mathbf{b}_z) \\
\mathbf{r}_t &= \sigma(\mathbf{W}_r \mathbf{x}_t + \mathbf{U}_r \mathbf{h}_{t-1} + \mathbf{b}_r) \\
\tilde{\mathbf{h}}_t &= \tanh(\mathbf{W}_h \mathbf{x}_t + \mathbf{U}_h (\mathbf{r}_t \odot \mathbf{h}_{t-1}) + \mathbf{b}_h + \xi_t); \quad \xi_t \sim \mathcal{N}(0, 0.01) \\
\mathbf{h}_t &= (1 - \mathbf{z}_t) \odot \mathbf{h}_{t-1} + \mathbf{z}_t \odot \tilde{\mathbf{h}}_t,
\end{align}
where $\eta$, $\mathbf{W}_z$, $\mathbf{U}_z$, $\mathbf{b}_z$, $\mathbf{W}_r$, $\mathbf{U}_r$, $\mathbf{b}_r$, $\mathbf{W}_h$, $\mathbf{U}_h$, $\mathbf{b}_h$ are trainable parameters. The latent used in the Results section ($\boldsymbol{z}$) is the hidden unit activity $h$. After model training, the NoisyGRU units are subsampled, centered, normalised, and rectified to give synthetic neural firing rates - which are $\boldsymbol{r}$ of the Results section. These firing rates are used to define a stochastic Poisson process to generate the synthetic neural data.

\noindent
\textbf{Students}

The student models are sequential autoencoders (SAEs) consisting of a bidirectional GRU that predicts the initial latent state, a Neural ODE (NODE) that evolves the latent dynamics (together these form the encoder, $f$, under our notation), and a linear readout layer mapping the latent states to the data (the decoder, $g$). We train several randomly initialised models with a range of latent dimensionalities ($3$, $5$, $8:16$, $32$, $64$). Models are trained to minimise a Poisson negative loglikelihood reconstruction loss, using the Adam \cite{kingma_adam_2017} optimiser.

\subsection*{Student-teacher Hidden Markov Models (HMMs)}

We choose both student and teacher to be discrete-space, discrete-time Hidden Markov Models (HMMs). As a teacher model, they simulate two important properties of neural time-series data: its dynamical nature and its stochasticity. As a student model, they are perhaps the simplest LVM for time-series, yet they are expressive enough to capture real neural dynamics\footnote{$\mathcal{Q}$ of 0.29 for HMMs vs. 0.24 for GPFA and 0.35 for LFADS, on \texttt{mc\_maze\_20}}. The HMM has a state space $z\in \{1,2,\dots, M\}$, and produces observations (emissions in HMM notation) along neurons $\mX$, with a state transition matrix $\mA$, emission model $\mB$ and initial state distribution $\boldsymbol{\pi}$. More explicitly:
\begin{equation}\label{eq:HMM}
\begin{aligned}
A_{m,l} &= p(z_{t+1}=l|z_t=m) \ & \forall \  m,l \\
B_{m,n} &= p(x_{n,t}=1|z_t=m) \ &\forall \ m,n \\
\pi_{m} &= p(z_0=m) \ & \forall \ m
\end{aligned}
\end{equation}


The same HMM can serve two roles: a) data-generation by sampling from \eqref{eq:HMM} and b) inference of the latents from data on a trial-by-trial basis:
\begin{align} \label{eq:HMMposterior}
    \xi_{t,m}^{(i)} &= f_m((\mX_{:,\text{in}})^{(i)}) = p(z_t^{(i)}=m|(\mX_{:,\text{in}})^{(i)}),
\end{align}
i.e., \textit{smoothing}, computed exactly with the forward-backward algorithm \citep{barber_bayesian_2012}. Note that although $z$ is the latent state of the HMM, we use its posterior probability mass function $\boldsymbol \xi_t$ as the relevant intermediate representation because it reflects a richer representation of the knowledge about the latent state than a single discrete state estimate. To make predictions of the rates of held-out neurons for co-smoothing we compute:
\begin{align}\label{eq: bernoulli readout}
    R_{n,t}^{(i)} &= g_n(\boldsymbol \xi^{(i)}_t) = \sum_m B_{m,n}\xi^{(i)}_{t,m}.  
\end{align}

As a teacher, we constructed a 4-state model of a noisy chain $A_{m,l}\propto \mathbb I[l=(m+1)\mod M]+\epsilon$, with $\epsilon=1e-2$, $\pi=\frac{1}{M}$, and $B_{m,n}\sim \text{Unif}(0,1)$ sampled once and frozen (Fig.~\ref{fig:hmm graphs}, left). We generated a dataset of observations from this teacher (see Table \ref{tab:data dimensions}).

For each student, we evaluate $\langle\mathcal Q^k_\text{S}\rangle$. This involves estimating the bernoulli emission parameters $\hat \mB_{m,\text{$k$-out}}$, given the latents $\xi_{t,m}^{(i)}$ using \eqref{eq:bernoulli MLE} and then generating rate predictions for the $k$-out neurons using \eqref{eq: bernoulli readout}.

\noindent
\textbf{HMM training}

HMMs are traditionally trained with expectation maximisation, but they can also be trained using gradient-based methods. We focus here on the latter as these are used ubiquitously and apply to a wide range of architectures. We use an existing implementation of HMMs with differentiable parameters: \href{https://github.com/probml/dynamax.git}{dynamax} \cite{linderman_dynamax_2025} -- a library of differentiable state-space models built with \href{https://github.com/google/jax}{jax}. 

We seek HMM parameters $\theta:=(A,B^{[\text{in},\text{out}]},\pi)$ that minimise the negative log-likelihood loss, $L$ of the held-in and held-out neurons in the train trials:
\begin{align}
L(\theta; \mathcal X_{[\text{in},\text{out}]}^\text{train}) &= -\log p(\mathcal X_{[\text{in},\text{out}]}^\text{train};\theta)\\
&= \sum_{i\in\text{train}} -\log p\left(\left(X_{1:T,[\text{in},\text{out}]}\right)^{(i)};\theta \right)
\end{align}

To find the minimum we do full-batch gradient descent on $L$, using dynamax together with the Adam optimiser \citep{kingma_adam_2017} .

\noindent
\textbf{Decoding across HMM latents}
 
 Consider two HMMs $u$ and $v$, of sizes $M(u)$ and $M(v)$, both candidate models of a dataset $\mathcal X$. Following \eqref{eq:HMMposterior}, each HMM can be used to infer latents from the data, defining encoder mappings $f^u$ and $f^v$. These map a single trial $i$ of the data $(\mX_{:,\text{in}})^{(i)}\in \mathcal X$ to $(\boldsymbol \xi^{(i)}_t)_u$ and $(\boldsymbol \xi^{(i)}_t)_v$. 

Since HMM latents are probability mass functions, we do not do use linear regression to learn the mappings across model latents. Instead we perform a multinomial regression from $(\boldsymbol \xi^{(i)}_t)_u$ to $(\boldsymbol \xi^{(i)}_t)_v$.

\begin{align}
\boldsymbol p^{(i)}_t = h\left( \left(\boldsymbol {\xi}^{(i)}_t \right)_u\right)\\
h(\xi)=\sigma (W\boldsymbol \xi + \boldsymbol b)
\end{align}

where $W\in \mathbb R^{M(v) \times M(u)}$, $\boldsymbol b\in \mathbb R^{M(v)}$ and $\sigma$ is the softmax. During training we sample states from the target PMFs $(z_t^{(i)})_v \sim (\boldsymbol {\xi}^{(i)}_t)_v$ thus arriving at a more well-known problem scenario: classification of $M(v)$-classes. We optimize $W$ and $\boldsymbol b$ to minimise a cross-entropy loss to the target $(\hat z_t^{(i)})_v$ using the \texttt{fit()} method of  \texttt{sklearn.linear\_model.LogisticRegression}.

We define decoding error, as the average Kullback-Leibler divergence $D_{KL}$ between target and predicted distributions:
\begin{align}
\mathcal D_{u\rightarrow v}:=\frac{1}{S^\text{test}T} \sum_{i\in \text{test}}\sum_{t=1}^T  D_{KL}\left(\boldsymbol p^{(i)}_t,(\boldsymbol \xi^{(i)}_t)_v\right)
\end{align}

where $D_{KL}$ is implemented with \texttt{scipy.special.rel\_entr}. 

In section \ref{sec:does not guarantee} and Fig.~\ref{fig:hmms}, the data $X$ is sampled from a single teacher HMM, $\text{T}$, and we evaluate $\mathcal D_{\text{T}\rightarrow\text{S}}$ and $\mathcal D_{\text{S}\rightarrow\text{T}}$ for each student notated simply as $\text{S}$.

\subsection*{Analysis of LVMs without access to ground truth}\label{methods:SOTA}

 We denote the set of high co-smoothing models as those satisfying $\mathcal Q>0.034$ for Fig.~\ref{fig:hmm fewshot} and $\mathcal Q> 0.8\times\mathcal Q_{\text{best model}}$ in Fig.~\ref{fig:SOTA_colsums}, $\mathcal F:= \{(f_u,g_u)\}_{u=1}^U$, the the encoders and decoders respectively. Note that STNDT is a deep neural network given by the composition $g \circ f$, and the choice of intermediate layer whose activity is deemed the `latent' $\mZ$ is arbitrary. Here we consider $g$ the last 'read-out' layer and $f$ to represent all the layers up-to $g$. 
 
\noindent
\textbf{Few-shot co-smoothing}

To perform few-shot co-smoothing, we learn $g'$, which takes the same form as $g$, a Poisson Generalised Linear Model (GLM) for each held-out neuron. We use \texttt{sklearn.linear\_model.PoissonRegressor}, which has a hyperparameter \texttt{alpha}, the amount of l2 regularisation. For the results in the main text, $\langle\mathcal Q^{k\text{-shot}}\rangle$ in Fig.~\ref{fig:SOTA_colsums}, we select $\alpha=10^{-3}$. We partition the training data into several random subsets of $k$ trials and train an independently initialised GLM on each subset. Each GLM is then evaluated on a fixed test set of trials (Fig.~\ref{fig:fewshot framework}), yielding a score for each subset. We report the mean over $\lfloor 5\times S^\text{train}/k \rfloor$ such repetitions, $\langle\mathcal Q^{k\text{-shot}}\rangle$, along with the standard error of the mean (error bars in Fig.~\ref{fig:hmm fewshot}, Fig.~\ref{fig:SOTA_colsums}). Scores are more variable at small $k$, so we need more repetitions to better estimate the average score. To implement this in a standarised way, we incorporate this chunking of data into several subsets in the \texttt{nlb\_tools} library (\nameref{sec: code}). This way we ensure that all models are trained and tested on identitical subsets. We report the compute-time for few-shot co-smoothing in \nameref{sec: fewshot compute time}.

\noindent
\textbf{Cross-decoding}

We perform a cross-decoding from the latents of model $u$, $(\mZ_{t,:})_u$, to those of model $v$, $(\mZ_{t,:})_v$, for every pair of models $u$ and $v$ using a linear mapping $h(\boldsymbol z) := W\boldsymbol z + \boldsymbol b$ implemented with \texttt{sklearn.linear\_model.LinearRegression}:

\begin{align}
    \left(\hat \mZ^{(i)}_{t,:}\right)_v = h_{ u\rightarrow v}\left(\left( \mZ^{(i)}_{t,:}\right)_u\right)
\end{align}

minimising a mean squared error loss. We then evaluate a $R^2$ score (\texttt{sklearn.metrics.r2\_score}) of the predictions, $(\hat\mZ)_v$, and the target, $(\mZ)_v$, for each mapping. We define the decoding error $\mathcal D_{u\rightarrow v}:=1-(R^2)_{u\rightarrow v}$. The results are accumulated into a $U\times U$ matrix (see Fig.~\ref{fig:cross decoding}).

\noindent
\textbf{Cycle consistency}

We evaluate cycle-consistency \cite{versteeg_expressive_2024,versteeg_computation-through-dynamics_2025} for a model $u$ also using a linear mapping from its rate predictions $\mR$ back to its latents $\hat \mZ$  implemented with \texttt{sklearn.linear\_model.LinearRegression}:

\begin{align}
    \left(\hat \mZ^{(i)}_{t,:}\right)_u = h_{\boldsymbol {r}\rightarrow \boldsymbol{\hat z}}\left(\left( \mR^{(i)}_{t,\text{out}}\right)_u\right),
\end{align}
again minimising a squared error loss. As in cross-decoding we evaluate $R^2$ score (\texttt{sklearn.metrics.r2\_score}) and the decoding error $\mathcal D_{\boldsymbol r\rightarrow \boldsymbol {\hat z}}:=1-(R^2)_{\boldsymbol r\rightarrow \boldsymbol {\hat z}}$ (Fig.~\ref{fig:cross decoding}A).

\subsection*{Summary of Neural Latent Benchmark (NLB) datasets}

Here are brief descriptions of the datasets used in this study. All datasets were collected from macaque monkeys performing sensorimotor or cognitive tasks. More comprehensive details can be found in the Neural Latents Benchmark paper \cite{pei_neural_2021}.

\begin{description}
  \item[\texttt{mc\_maze} \cite{churchland_cortical_2010}] 
  Motor cortex recordings during a delayed reaching task where monkeys navigated around virtual barriers to reach visually cued targets. The task involved 108 unique maze configurations, with several repeated trials for each one, thus serving as a "neuroscience MNIST". We choose this dataset to visualise the latents in Fig.~\ref{fig:SOTA_colsums}.

  \item[\texttt{mc\_rtt} \cite{odoherty_nonhuman_2018}] 
  Motor cortex recordings during naturalistic, continuous reaching toward randomly appearing targets without imposed delays. The task lacks trial structure and includes highly variable movements, emphasizing the need for modeling unpredictable inputs and non-autonomous dynamics.

  \item[\texttt{dmfc\_rsg} \cite{sohn_bayesian_2019}] 
  Recordings from dorsomedial frontal cortex during a time-interval reproduction task, where monkeys estimated and reproduced time intervals between visual cues using eye or hand movements. The task involves internal timing, variable priors, and mixed sensory-motor demands.

  \item[\texttt{area2\_bump} \cite{chowdhury_area_2020}] 
  Somatosensory cortex recordings during a visually guided reach task in which unexpected mechanical bumps to the limb occurred in half of the trials. The task probes proprioceptive feedback processing and requires modeling input-driven neural responses.
\end{description}

\subsection*{Dimensions of datasets}

We analyse several datasets in this work. Three synthetic datasets generated by an RNN, HMM (Methods, Fig.~\ref{fig:hmm graphs}) and LGSMM (\nameref{sec: supp LGSSM}) and the four datasets from the Neural Latent Benchmarks (NLB) suite \cite{pei_neural_2021,churchland_cortical_2010,odoherty_nonhuman_2018,sohn_bayesian_2019,chowdhury_area_2020}. In table \ref{tab:data dimensions}, we summarise the dimensions of all these datsets. To evaluate $k$-shot on the existing SOTA methods while maintaining the NLB evaluations, we conserved the \textit{forward-prediction} aspect. During model training, models output rate predictions for $T^\text{fp}$ future time bins in each trial, i.e., \eqref{eq:encoder} and \eqref{eq:time-step constraint} are evaluated for $1 \leq t \leq T^\text{fp}$ while input remains as $\mX_{1:T,\text{in}}$. Although we do not discuss the forward-prediction metric in our work, we note that the SOTA models receive gradients from this portion of the data.

In all the NLB datasets as well as the RNN dataset we reuse held-out neurons as $k$-out neurons. We do this to preserve NLB evaluation metrics on the SOTA models, as opposed to re-partitioning the dataset resulting in different scores from previous works. This way existing co-smoothing scores are preserved and $k$-shot co-smoothing scores can be directly compared to the original co-smoothing scores. The downside is that we are not testing the few-shot on `novel' neurons. Our numerical results (Fig.~\ref{fig:SOTA_colsums}) show that our concept still applies.

\begin{table}[h!] 
    \caption{Dimensions of real and synthetic datasets. Number of train and test trials $S^{\text{train}}$, $S^{\text{test}}$, time-bins per trial for co-smoothing $T$, and forward-prediction $T^\text{fp}$, held-in, held-out and $k$-out neurons $N^{\text{in}}$, $N^{\text{out}}$, $N^{\text{$k$-out}}$. \textsuperscript{\dag}In all the NLB\cite{pei_neural_2021} datasets as well the RNN dataset we use the same set of neurons for $N^{\text{out}}$ and $N^{\text{$k$-out}}$.}
    \label{tab:data dimensions}
    \centering
    \vskip 0.15in
\begin{center}
\begin{small}
\begin{sc}
    \begin{tabular}{p{5cm} c c c c c c c}
    \toprule
        Dataset & $S^{\text{train}}$ & $S^{\text{test}}$ & 
        $T$ & $T^\text{fp}$ &
        $N^{\text{in}}$ & $N^{\text{out}}$ & $N^{\text{$k$-out}}$
        \\
        \midrule
        Synthetic Noisy GRU RNN (Methods) \cite{versteeg_computation-through-dynamics_2025} & $800$ & $200$ & $500$ & -- & $50$ & $10$ & $10$\textsuperscript{\dag} 
        \\
        Synthetic HMM (Methods) & $2000$ & $100$ & $10$ & -- & $20$ & $50$ & $50$ 
        \\
        Synthetic LGSSM (\nameref{sec: supp LGSSM}) & $20$ & $500$ & $10$ & -- & $5$ & $30$ & $30$
        \\
        \texttt{mc\_maze\_20} \cite{churchland_cortical_2010} & $1721$ & $574$ & $35$ & $10$ & $137$ & $45$ & $45$\textsuperscript{\dag}
        \\
        \texttt{mc\_rtt\_20} \cite{odoherty_nonhuman_2018} & $810$ & $270$ & $30$ & $10$ & $98$ & $32$ & $32$\textsuperscript{\dag}
        \\
        \texttt{dmfc\_rsg\_20} \cite{sohn_bayesian_2019} & $748$ & $258$ & $75$ & $10$ & $40$ & $14$ & $14$\textsuperscript{\dag}
        \\
        \texttt{area2\_bump\_20} \cite{chowdhury_area_2020} & $272$ & $92$ & $30$ & $10$ & $49$ & $16$ & $16$\textsuperscript{\dag}
        \\
        \bottomrule
    \end{tabular}
    \end{sc}
\end{small}
\end{center}
\vskip -0.1in
\end{table}

\input{methods_theory/theoretical_analysis_few_shot_learning_hmm}
\input{methods_theory/theoretical_analysis_ridgeless_least_squares_regression}
\input{methods_theory/theoretical_analysis_prototype_learning}

\newpage
\appendix
\onecolumn

\section*{Supplementary  information}

 \input{appendices_new_with_figures}



\bibliography{references,example_paper}

%
%
%
%

\end{document}

%% file: math_commands.tex
\usepackage{amsmath,amsfonts,bm}

\def\eqref#1{equation~\ref{#1}}

\def\1{\bm{1}}

\def\vmu{{\bm{\mu}}}

\def\vb{{\bm{b}}}
\def\vc{{\bm{c}}}

\def\vw{{\bm{w}}}
\def\vx{{\bm{x}}}

\def\vz{{\bm{z}}}

\def\mA{{\bm{A}}}
\def\mB{{\bm{B}}}

\def\mF{{\bm{F}}}
\def\mG{{\bm{G}}}
\def\mH{{\bm{H}}}
\def\mI{{\bm{I}}}

\def\mR{{\bm{R}}}

\def\mX{{\bm{X}}}

\def\mZ{{\bm{Z}}}

\def\mSigma{{\bm{\Sigma}}}

\DeclareMathAlphabet{\mathsfit}{\encodingdefault}{\sfdefault}{m}{sl}
\SetMathAlphabet{\mathsfit}{bold}{\encodingdefault}{\sfdefault}{bx}{n}

\def\sR{{\mathbb{R}}}

\def\sZ{{\mathbb{Z}}}

\def\emR{{R}}

\def\emX{{X}}

\newcommand{\Cov}{\mathrm{Cov}}

\DeclareMathOperator*{\argmin}{arg\,min}

\DeclareMathOperator{\Tr}{Tr}

%% file: methods_theory/theoretical_analysis_few_shot_learning_hmm.tex
\subsection*{Theoretical analysis of few shot learning in HMMs.}
\label{sec: hmm analytical}

Consider a student-teacher scenario as in section \ref{sec:does not guarantee}. We let $T=2$ and use a stationary teacher $z^{(i)}_1=z^{(i)}_{2}$. Now consider two examples of inferred students. To ensure a fair comparison, we use two latent states for both students. In the \textit{good} student, $\xi$, these two states statistically do not depend on time, and therefore it does not have extraneous dynamics. In contrast, the \textit{bad} student, $\mu$, uses one state for the first time step, and the other for the second time step. A particular example of such students is:

\begin{eqnarray}
\xi_{t} = \begin{bmatrix}
    0.5 & 0.5
\end{bmatrix}^T \ t \in  \{1,2\}\\
\begin{array}{ccc}
     \mu_{t=1} = 
    \begin{bmatrix}
    1 & 0
\end{bmatrix}^T & & \mu_{t=2}=\begin{bmatrix} 
    0 & 1
\end{bmatrix}^T
\end{array} 
\end{eqnarray}

where each vector corresponds to the two states, and we only consider two time steps.


We can now evaluate the maximum likelihood estimator of the emission matrix from $k$ trials for both students. In the case of bernoulli HMMs the maximum likelihood estimate of $g'$ given a fixed $f$ and $k$ trials has a closed form:

\begin{align}\label{eq:bernoulli MLE}
    \hat B_{m,n} &= \frac{\sum_{i\in \text{$k$-shot trials}} \sum_{t=1}^T \mathbb I[X^{(i)}_{t,n}=1]\xi^{(i)}_{t,m}}{\sum_{i'\in \text{$k$-shot trials}} \sum_{t'=1}^T  \xi^{(i')}_{t',m}} & \forall \ 1\leq m \leq M \text{ and } n\in \text{$k$-out neurons} 
\end{align}

We consider a single neuron, and thus omit $n$, reducing the estimates to:

\begin{equation}
\begin{array}{ccc}
\hat{B}_1(\xi) = \frac{0.5 (C_1 + C_2)}{0.5kT}& & \hat{B}_{1}(\mu) = \frac{C_1}
{k}\\
\hat{B}_{2}(\xi) = \frac{0.5 (C_1 + C_2)}{0.5kT}& & \hat{B}_{2}(\mu) = \frac{C_2}{k}\\
\end{array}
\end{equation}

where $C_t$ is the number of times $x=1$ at time $t$ in $k$ trials. We see that $C_t$ is a sum of $k$ i.i.d. Bernoulli random variables (RVs) with the teacher parameter $B^*$, for both $t=1,2$.

Thus, $\hat{B}_{m}(\xi)$ and $\hat{B}_{m}(\mu)$ are scaled binomial RVs with the following statistics: 


\begin{equation}\label{eq:hat B statistics}
\begin{array}{ccc}
\mathbb E \hat{B}_{1}(\xi)=\mathbb E \hat{B}_{2}(\xi) =B^*& & \mathbb E \hat{B}_{1}(\mu)=\mathbb E \hat{B}_{2}(\mu) =B^*\\
\Cov \left[\hat{\mB}(\xi) \right] =\frac{1}{2k}B^*(1-B^*) \begin{bmatrix}
    1 & 1 \\ 1 & 1
\end{bmatrix} & & \Cov \left[\hat{\mB}(\mu) \right] =\frac{1}{k}B^*(1-B^*) \begin{bmatrix}
    1 & 0 \\ 0 & 1
\end{bmatrix} 
\end{array}
\end{equation}

The test loss is given by $L(\hat B)=\mathbb E \frac{1}{T}\sum_t\log p(X^{(i)}_t;\hat B) = \frac{1}{T}\sum_t B^* \log\left(R_t\right)+ (1-B^*)\log\left(1-R_t\right)$. For $\xi$, $R_t=0.5(\hat B_1+\hat B_2)$ for both values of $t$, and for $\mu$,  $R_1=\hat B_1$ and $R_2=\hat B_2$. Ultimately,

\begin{align}
L_\xi(\hat \mB(\xi)) &= \frac{1}{T}\sum_t B^* \log\left(0.5(\hat B_1+\hat B_2)\right)+ (1-B^*)\log\left(1-0.5(\hat B_1+\hat B_2)\right)\\
L_\mu(\hat \mB(\mu)) &= \frac{1}{T}\sum_t B^* \log\left(\hat B_t\right)+ (1-B^*)\log\left(1-\hat B_t\right)
\end{align}

To see how these variations affect the test loglikelihood $L$ of the few-shot regression on average, we do a taylor expansion around $B^*$, recognising that the function is maximised at $B^*$, so $\frac{\partial L}{\partial \mB}\Big\vert_{B^*}=0$.

\begin{align}
\mathbb E_{\hat B_k} L(\hat B_k)&=\mathbb E_{\hat B_k}\left[ L(B_\infty)+\frac{1}{2}(\hat B_k-B^*)^T\frac{\partial^2 L}{\partial B^2}\bigg\vert_{B^*} (\hat B_k-B^*) + \dots \right] \\
&\approx L(B^*)+\mathbb E_{\hat B_k} \frac{1}{2}(\hat B_k-B^*)^T \frac{\partial^2 L}{\partial B^2}\bigg\vert_{B^*} (\hat B_k-B^*) \\
&= L(B^*)+ \underbrace{\frac{1}{2}(\mathbb E\hat B_k-B^*)^T \frac{\partial^2 L}{\partial B^2}\bigg\vert_{B^*}(\mathbb E\hat B_k-B^*)}_{\text{bias}}+ \underbrace{\frac{1}{2}\text{Tr}\left[\Cov(\hat B_k) \frac{\partial^2 L}{\partial B^2}\bigg\vert_{B^*}\right]}_{\text{variance}} \label{eq:HMM bias variance}
\end{align}

We see that this second order truncation of the loglikelihood is decomposed into a bias and a variance term. We recognise that the bias term goes to zero because we know the estimator is unbiased (\eqref{eq:hat B statistics}). To compute the variance term, we compute the hessians which differ for the two models:

\begin{align}
    \frac{\partial^2 L_\xi}{\partial B^2}\bigg\vert_{B^*}=-\frac{\eta}{4}\begin{bmatrix}1 & 1\\ 1&1\end{bmatrix}, & & \frac{\partial^2 L_\mu}{\partial B^2}\bigg\vert_{B^*} = -\frac{\eta}{2}\begin{bmatrix}1 & 0\\ 0&1\end{bmatrix},
\end{align}

where $\eta=\frac{1}{B^*(1 - B^*)}$.

Incorporating these hessians into \eqref{eq:HMM bias variance}, we obtain:
\begin{align}
    \mathbb E_{\hat B_k} L_{\xi}\left(\hat B_k\left(\xi\right)\right) &\approx L(B^*) - \frac{1}{8k} \Tr\begin{bmatrix}
        2 & 2\\ 2 & 2
    \end{bmatrix} = L(B^*) - \frac{1}{2k},\label{eq:HMM analytical1}\\
    \mathbb E_{\hat B_k} L_{\mu}\left(\hat B_k\left(\mu\right)\right) &\approx L(B^*) - \frac{1}{2k} \Tr\begin{bmatrix}
        1 & 0\\ 0 & 1
    \end{bmatrix} = L(B^*) - \frac{1}{k}.\label{eq:HMM analytical2} 
\end{align}

Fig.~\ref{fig:kshot_theory}A shows these analytical results against the left hand side of \eqref{eq:HMM analytical1} and \eqref{eq:HMM analytical2} evaluated numerically.


%% file: methods_theory/theoretical_analysis_ridgeless_least_squares_regression.tex
\subsection*{Theoretical analysis of ridgeless least squares regression with extraneous noise.}
\label{sec: ridgeless analysis}

Teacher latents $z_i^*\sim \mathcal N(0,1)$ generate observations $x_i$:
\begin{align}
x_i = z^*_i + \epsilon_i,
\end{align}
where $\epsilon_i\sim \mathcal N(0,\sigma_\text{obs}^2)$ is observation noise.

In this setup there is no time index: we consider only a single sample index $i$.

We consider candidate student latents, $\vz\in \sR^p$, that \emph{contain} the teacher along with extraneous noise, i.e:
\begin{align}
    \vz_i := \begin{bmatrix}
        z^*_i & \boldsymbol{\xi}_i
    \end{bmatrix}^T,
\end{align}
where $\boldsymbol{\xi}_i\sim\mathcal N(0,\sigma_\text{ext}^2 \mI_{p-1})$ is a vector of i.i.d. extraneous noise, and $\mI_{p-1}$ is the $(p-1)\times (p-1)$ identity matrix.

We study the minimum $l_2$ norm least squares regression estimator on $k$ training samples:
\begin{align}
    \hat \vw = \argmin \left\{ \Vert w \Vert_2: w \text{ minimises } \sum_{i=1}^k\Vert x_i-\vw^T\vz_i \Vert^2_2 \right\}.
\end{align}
with the regression weights $\vw\in\mathbb \sR^p$. More succinctly, $\vz_i\sim\mathcal N(0,\Sigma)$, where $\Sigma=\diag([1,\sigma_\text{ext}^2,\dots,\sigma_\text{ext}^2])$.

Note that, by construction, the true mapping is: 

\begin{align}
\vw^*=\begin{bmatrix}1 & 0 & \dots & 0\end{bmatrix}^T. 
\end{align}

Test loss or \emph{risk} is a mean squared error:
\begin{align}
R(\hat \vw;\vw^*) = \mathbb E_{\vz_0} \left(\vz_0^T\vw^*- \vz_0^T\hat \vw\right)^2,
\end{align}

given a test sample $\vz_0$. The error can be decomposed as:

\begin{align}
R(\hat \vw;\vw^*) = \underbrace{\left\Vert \mathbb E(\hat \vw) - \vw^* \right\Vert^2_\Sigma}_\text{bias, $B$} + \underbrace{\Tr\left[ \Cov(\hat \vw )\Sigma\right]}_\text{variance, $V$},
\end{align}

The scenario described above is a special case of \cite{hastie_surprises_2022}. What follows is a direct application of their theory, which studies the risk $R$, in the limit $k,p\to \infty$ such that $p/k\to \gamma \in (0,\infty)$, to our setting.  The alignment of the theory with numerical simulations is demonstrated in Fig.~\ref{fig:kshot_theory}B.

\begin{claim}{$\gamma<1$, i.e., the underparameterised case $k>p$.}\label{claim: hastie underparam}

$B=0$ and the risk is just variance and is given by:
\begin{align}
\lim_{k,p\to\infty \text{ and } p/k\to\gamma} R(\hat \vw;\vw^*)=\sigma_\text{obs}^2\frac{\gamma}{1-\gamma},
\end{align}
with no dependence on $\sigma_\text{ext}$. 
\end{claim}
\begin{proof}
This is a direct restatement of Proposition 2 in \cite{hastie_surprises_2022}. 
\end{proof}

\begin{claim}{$\gamma>1$, i.e., the overparameterised case $k<p$.}\label{claim: hastie overparam}

The following is true as $k,p\to\infty \text{ and } p/k\to\gamma$:
\begin{align}
\lim_{k,p\to\infty \text{ and } p/k\to\gamma} B &= \frac{\gamma (\gamma-1)}{\left(\gamma-1 + \frac{1}{\sigma_\text{ext}^2}\right)^2}\\
\lim_{k,p\to\infty \text{ and } p/k\to\gamma} V &= \sigma_\text{obs}^2\frac{\gamma}{\gamma-1}
\end{align}
\end{claim} 
\begin{proof}

For the non-isotropic case \cite{hastie_surprises_2022} define the following distributions based on the eigendecomposition of $\Sigma$.

\begin{align}
d\widehat H(s) &= \frac{1}{p} \delta(s-1) + \frac{p-1}{p} \delta(s-\sigma_\text{ext}^2)\\
d\widehat G(s) &=  \delta(s-1)
\end{align}

In the limit $p\to\infty$ we take $d\widehat H(s)\approx\delta(s-\sigma_\text{ext}^2)$. This greatly simplifies calculations and nevertheless provide a good fit for numerical results with finite $k$ and $p$.
We solve for $c_0(\gamma,\widehat H)$ using equation 12 in \cite{hastie_surprises_2022}.

\begin{align}
\gamma c_0 = \frac{1}{(\gamma-1 )\sigma_\text{ext}^2}
\end{align}

We then compute the limiting values of $B$ and $V$:
\begin{align}
B &= \Vert w^* \Vert^2 (1+
\gamma c_0 \sigma_\text{ext}^2)\frac{1}{(1+\gamma c_0)^2}\\
V &= \sigma_\text{obs}^2\gamma c_0 \sigma_{ext}^2.
\end{align}

Substituting $\gamma c_0$ completes the proof.
\end{proof}

The extraneous noise, $\sigma_\text{ext}$, influences the risk of ridgeless regression only in the regime $k < p$, and its effect is confined to the bias term, leaving the variance unaffected. In contrast, observation noise contributes exclusively to the variance term. Consequently, the dependence of the risk on $\sigma_{\text{ext}}$ persists even in the absence of observation noise, i.e., when $\sigma_{\text{obs}} = 0$. 

Fig.~\ref{fig:kshot_theory}B presents the theoretical predictions alongside the empirical average $k$-shot performance of minimum-norm least-squares regression, computed numerically using the function \texttt{numpy.linalg.lstsq}.


%% file: methods_theory/theoretical_analysis_prototype_learning.tex
\subsection*{Theoretical analysis of prototype learning for binary classification with extraneous noise.}
\label{sec: prototype learning analysis}

Teacher latents are distributed as $p(z_i^*)=\frac{1}{2}\delta(z_i^*-\frac{1}{\sqrt 2})+ \frac{1}{2}\delta(z_i^*+\frac{1}{\sqrt 2})$, that is either $\frac{1}{\sqrt 2}$ or $-\frac{1}{\sqrt 2}$ with probability $\frac{1}{2}$, representing two classes $a$ and $b$ respectively. 

We consider candidate student latents, $\vz\in \sR^{2M+1}$, that \emph{contain} the teacher along with extraneous noise, i.e:
\begin{align}
    \vz_i := \begin{cases}
    \begin{bmatrix}
        z^*_i & \boldsymbol{\xi}_i & \boldsymbol 0
    \end{bmatrix}^T & \text{if $z_i^*=1$}\\
    \begin{bmatrix}
        z^*_i & \boldsymbol 0 & \boldsymbol{\xi}_i
    \end{bmatrix}^T & \text{if $z_i^*=-1$}
    \end{cases}
\end{align}
where $\boldsymbol{\xi}_i\sim\mathcal N(0,\sigma_\text{ext}^2 \mI_{M})$ is a $M$-vector of i.i.d. extraneous noise, and $\mI_{M}$ is the $M\times M$ identity matrix and $\boldsymbol{0}\in \sR^M$.

We consider the prototype learner $\vw=\bar\vz_\text{a}-\bar\vz_\text{b}$, $\vb=\frac{1}{2}(\bar\vz_\text{a}+\bar\vz_\text{b})$, where $\bar \vz_\text{a}$ and $\bar \vz_\text{b}$ are the sample means of $k$ latents from class $\text{a}$ and $k$ latents from class $\text{b}$ respectively. The classification rule is given by the sign of $\vw^T\vx-\vb$: classifying the input $\vx$ as $\text{a}$ if positive and $\text{b}$ otherwise.

This setting is a special case of \cite{sorscher2022fewshot}. They provide a theoretical prediction for average few-shot classification error rate for class $\text{a}$, $\epsilon_\text{a}$, given by $\epsilon_\text{a} = H(SNR)$ where $H(x)=\frac{1}{\sqrt{2\pi}}\int_x^\infty dt \exp(-t^2/2)$ is a monotonously decreasing function.

\begin{align}
    SNR_\text{a} &= \frac{1}{2} \frac{\Vert \Delta \vx_0\Vert^2 + (R_\text{b}^2 R_\text{a}^2-1)/k}{\sqrt{D_\text{a}^{-1}/k+\Vert \Delta \vx_0^T U_\text{b}\Vert^2/k + \Vert \Delta \vx_0^T U_\text{a} \Vert^2}}.
\end{align}

$\Delta \vz= \vz_\text{a}-\vz_\text{b}$ the difference of the population centroids of the two classes.

In our case this reduces to:
\begin{align}
    SNR &\approx \frac{\sqrt{Mk}}{\sigma_\text{ext}^2} \label{eq: sorscher SNR simplified}
\end{align}

To obtain this we note radii of manifold $\text{a}$ is $\begin{bmatrix}
    0 & \sigma_\text{ext} & \dots & \sigma_\text{ext} & 0 & \dots &  0 
\end{bmatrix}$ with an average radius $R=R_a=R_b=\frac{M}{(2M+1)}\sigma_\text{ext}^2$ and participation ratio $D_a = \left(\sum_i (R_a^i)^2\right)^2/\sum_i( R_a^i)^4=M$.

We substitute $\Vert \Delta \vx_0\Vert^2=\frac{1}{R^2} = \frac{2M+1}{M\sigma_\text{ext}^2}\approx \frac{2}{\sigma_\text{ext}^2}$.

The bias term $(R_\text{b}^2 R_\text{a}^2-1)/k$ is zero since $R_\text{a} = R_\text{b}$.

The $\Delta \vx_0^T U_\text{a}$ and $\Delta \vx_0^TU_\text{b}$ terms are both zero.

The participation ratio $D_a=M$.
Our construction is symmetric in that $SNR_\text{a}=SNR_\text{b}$.

The classification error, $\epsilon$, decreases monotonically with the number of samples $k$, tending to zero as $k \to \infty$ for all finite values of $\sigma_{\text{ext}}$. In contrast, $\epsilon$ increases monotonically with extraneous noise $\sigma_{\text{ext}}$, deviating significantly from zero once $\sigma_{\text{ext}}^{2} \approx \sqrt{Mk}$. 

Fig.~\ref{fig:kshot_theory}C presents the numerically computed error in comparison with the theoretical prediction given in \eqref{eq: sorscher SNR simplified}.


%% file: appendices_new_with_figures.tex
\paragraph*{S1~Fig.}
\input{supporting_information/student_teacher_rnns_cosmoothing_model_size}
\paragraph*{S2~Fig.}
\input{supporting_information/how_to_choose_k_for_dataset}
\paragraph*{S1~Appendix.}
\input{supporting_information_figsandtablesremoved/decoding_across_hmm_latents}
\paragraph*{S3~Fig.}
\input{supporting_information/good_cosmoothing_does_not_guarantee_correct_latents_hmm}
\paragraph*{S4~Fig.}
\input{supporting_information/student_teacher_results_linear_gaussian_state_space_models}
\paragraph*{S5~Fig.}
\input{supporting_information/hmm_network_visualisations}
\paragraph*{S6~Fig.}
\input{supporting_information/fewshot_not_simply_hard_cosmoothing}
\paragraph*{S2~Appendix.}
\input{supporting_information_figsandtablesremoved/code_repositories}
\paragraph*{S3~Appendix.}
\input{supporting_information_figsandtablesremoved/time_cost_computing_fewshot_cosmoothing}

%% file: supporting_information/student_teacher_rnns_cosmoothing_model_size.tex
{\bf Student-Teacher RNNs: co-smoothing as a function of model size.}
\label{sec:co-smoothing vs latent size RNNs}

Finding the correct model is not just about tuning the latent size hyperparameter. Models over a range of sizes achieve high co-smoothing (Fig.~\ref{fig:co-smoothing vs latent size RNNs}).

\begin{figure}
    \centering
    \includegraphics[width=\linewidth]{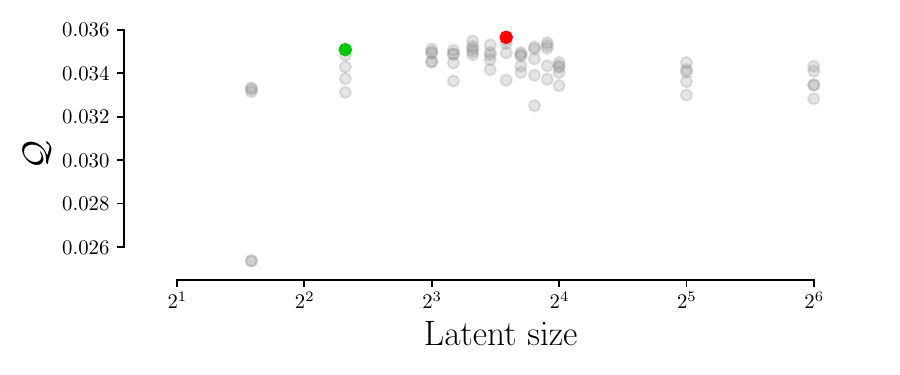}
    \caption{\textbf{Performance and model size.} Co-smoothing performance $\mathcal Q$ as function of the model size for the NODE SAE students with different dimensionalities on the same 64-unit noisy GRU performing 3BFF (Methods). Models of sizes in the range 5-15 yield the highest performance.}
    \label{fig:co-smoothing vs latent size RNNs}
\end{figure}

%% file: supporting_information/how_to_choose_k_for_dataset.tex
{\bf How to choose $k$ for your dataset?}
\label{sup:how to choose k}

The main text section on why few-shot works reveals that extraneous models are best discriminated when the shot number, $k$, is small (Fig.~\ref{fig:kshot_theory}). So how small can we go? In the case of sparse data like neural spike counts we may obtain $k$-trial subsets in which some neurons are silent. In this scenario the few-shot decoder $g'$ receives no signal for those neurons. To avoid this pathological scenario, for each dataset, we pick the smallest possible $k$ that ensures that the probability of encountering silent neurons in a $k$-trial subset is safely near zero. This must be computed for each dataset independently since some datasets are more sparse than others. In Fig.~\ref{fig:zero_spike_fractions} we compute the frequency of such silences for different $k$, for each NLB\cite{pei_neural_2021} dataset, and show the values of $k$ chosen for the analysis in the main text.

\begin{figure}
    \centering
    \includegraphics[width=\linewidth]{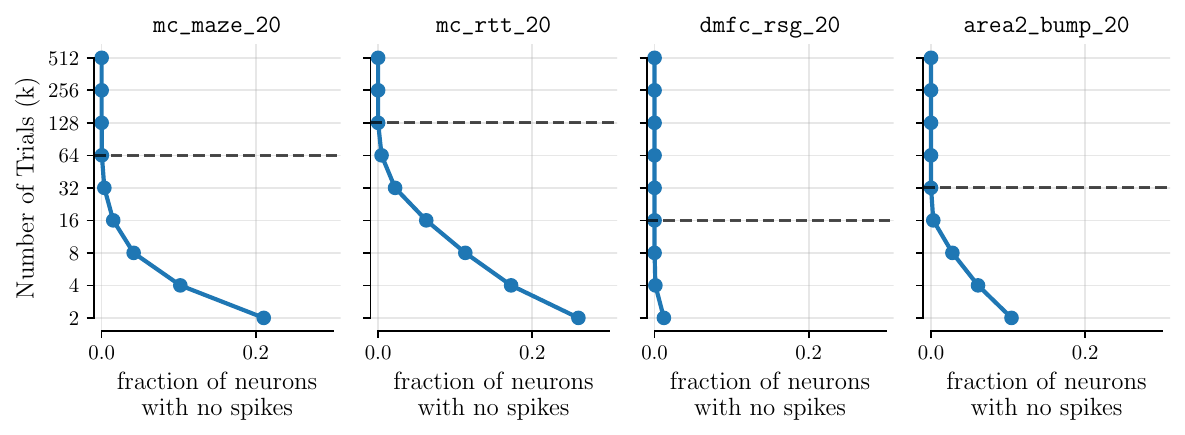}
    \caption{Choosing $k$ for each of the NLB datasets. Our theoretical analysis suggests taking $k$ to be as small as possible (main text Fig.~\ref{fig:fewshot framework}) to maximise the difference in $k$-shot performance. However, in practice, we would like to avoid training on subsets with totally silent neurons. Thus we pick the smallest $k$ that avoids totally silent neurons (dashed lines). This has to be checked for each dataset given its particular distribution of firing rates.}
    \label{fig:zero_spike_fractions}
\end{figure}

%% file: supporting_information_figsandtablesremoved/decoding_across_hmm_latents.tex
{\bf Decoding across HMM latents: fitting and evaluation.}
\label{sec:HMM decoding}
Consider two HMMs $u$ and $v$, of sizes $M(u)$ and $M(v)$, both candidate models of a dataset $\mathcal X$. Following \eqref{eq:HMMposterior}, each HMM can be used to infer latents from the data, defining encoder mappings $f^u$ and $f^v$. These map a single trial $i$ of the data $(\mX_{:,\text{in}})^{(i)}\in \mathcal X$ to $(\boldsymbol \xi^{(i)}_t)_u$ and $(\boldsymbol \xi^{(i)}_t)_v$. 

We now perform a multinomial regression from $(\boldsymbol \xi^{(i)}_t)_u$ to $(\boldsymbol \xi^{(i)}_t)_v$.

\begin{align}
\boldsymbol p^{(i)}_t = h\left( \left(\boldsymbol {\xi}^{(i)}_t \right)_u\right)\\
h(\xi)=\sigma (W\boldsymbol \xi + \boldsymbol b)
\end{align}

where $W\in \mathbb R^{M(v) \times M(u)}$, $\boldsymbol b\in \mathbb R^{M(v)}$ and $\sigma$ is the softmax. During training we sample states from the target PMFs $(z_t^{(i)})_v \sim (\boldsymbol {\xi}^{(i)}_t)_v$ thus arriving at a more well know problem scenario: classification of $M(v)$-classes. We optimize $W$ and $\boldsymbol b$ to minimise a cross-entropy loss to the target $(\hat z_t^{(i)})_v$ using the \texttt{fit()} method of  \texttt{sklearn.linear\_model.LogisticRegression}.

We define decoding error, as the average Kullback-Leibler divergence $D_{KL}$ between target and predicted distributions:
\begin{align}
\mathcal D_{u\rightarrow v}:=\frac{1}{S^\text{test}T} \sum_{i\in \text{test}}\sum_{t=1}^T  D_{KL}\left(\boldsymbol p^{(i)}_t,(\boldsymbol \xi^{(i)}_t)_v\right)
\end{align}

where $D_{KL}$ is implemented with \texttt{scipy.special.rel\_entr}. 

In section \ref{sec:does not guarantee} and Fig.~\ref{fig:hmms}, the data $X$ is sampled from a single teacher HMM, $\text{T}$, and we evaluate $\mathcal D_{\text{T}\rightarrow\text{S}}$ and $\mathcal D_{\text{S}\rightarrow\text{T}}$ for each student notated simply as $\text{S}$.

%% file: supporting_information/good_cosmoothing_does_not_guarantee_correct_latents_hmm.tex
{\bf Good co-smoothing does not guarantee correct latents in Hidden Markov Models (HMMs).}
\label{sec: does not guarantee HMM}
In the main text, we show how good prediction of held-out neural activity, i.e., \textit{co-smoothing}, does not guarantee a match between model and true latents. We did this in the student-teacher setting of RNNs and NODE SAEs (Fig.~\ref{fig:hmm graphs}). Here we replicate the results in HMMs (see Methods) (Fig.~\ref{fig:HMM DTS DST}). The arrows mark the ``Good'' and ``Bad'' transition matrices shown in the Fig.~\ref{fig:hmm graphs} (lower).

\begin{figure}[ht!]
    \centering
    \includegraphics[width=\linewidth]{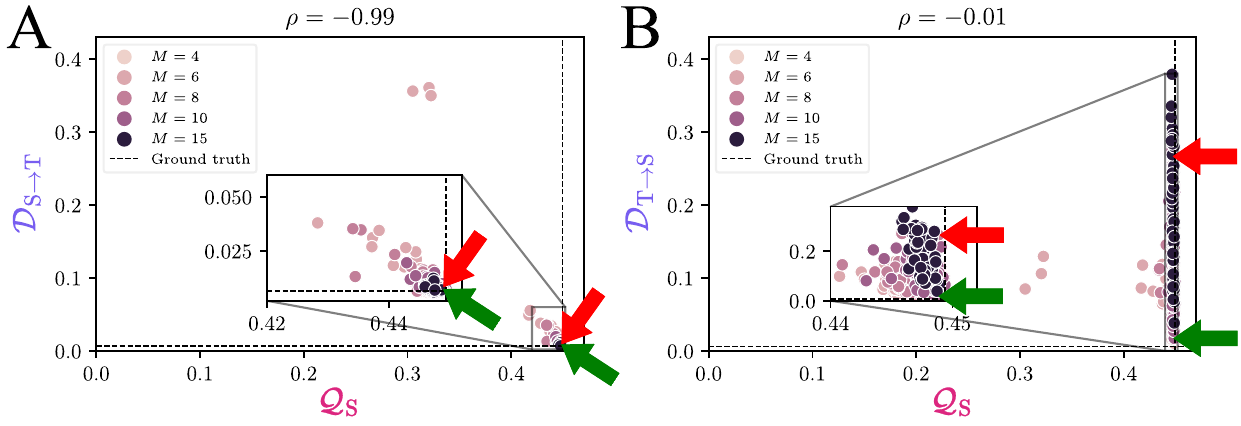}
    \caption{Similar to Fig.~\ref{fig:hmm graphs}, several students HMMs are trained on a dataset generated by a single teacher HMM, a noisy 4-cycle. The Student$\rightarrow$Teacher decoding error $\mathcal D_{\text{S}\rightarrow \text{T}}$ is low and tightly related to the co-smoothing score. The Teacher$\rightarrow$Student decoding error $\mathcal D_{\text{T}\rightarrow \text{S}}$ is more varied and uncorrelated to co-smoothing.}
    \label{fig:HMM DTS DST}
\end{figure}

%% file: supporting_information/student_teacher_results_linear_gaussian_state_space_models.tex
{\bf Student-teacher results in Linear Gaussian State Space Models.}
\label{sec: supp LGSSM}

We demonstrate that our results are not unique to the RNN or HMM settings by simulating another simple scenario: linear gaussian state space models (LGSSM), i.e., Kalman Smoothing.

The model is defined by parameters $(\vmu_0,\mSigma_0,\mF,\mG,\mH,\mR)$. A major difference to HMMs is that the latent states $\boldsymbol z \in \mathbb R^M$ are continuous. They follow the dynamics given by:

\begin{align}
\vz_0&\sim \mathcal N(\vmu_0,\mSigma_0)\\
\vz_t&\sim \mathcal N(\mF\vz_{t-1}+\vb, \mG) \\
\vx_t&\sim \mathcal N(\mH\vz_t+\vc, \mR)
\end{align}

Given these dynamics, the latents $\boldsymbol z$ can be inferred from observations $\boldsymbol x$ using Kalman smoothing, analogous to \eqref{eq:HMMposterior}. Here we use the jax based \href{https://github.com/probml/dynamax.git}{dynamax} implementation.

As with HMMs we use a teacher LGSSM with $M=4$, with parameters chosen randomly (using the dynamax defaults) and then fixed. Student LGSSMs are also initialised randomly and optimised with Adam \citep{kingma_adam_2017} to minimise negative loglikelihood on the training data (see 
table \ref{tab:data dimensions} for dimensions of the synthetic data set). $\mathcal D_{\text{S}\rightarrow\text{T}}$ and $\mathcal D_{\text{T}\rightarrow\text{S}}$ is computed with linear regression (\texttt{sklearn.linear\_model.LinearRegression}) and predictions are evaluated against the target using $R^2$ (\texttt{sklearn.metrics.r2\_score}). We define $\mathcal D_{u\rightarrow v}:=1-(R^2)_{u\rightarrow v}$. Few-shot regression from $\boldsymbol z$ to $\boldsymbol x^\text{$k$-out}$ is also performed using linear regression.

In line with our results with RNNs and HMMs (Fig.~\ref{fig:hmm graphs} and Fig.~\ref{fig:hmm fewshot}), in Fig.~\ref{fig:kshot-LGSSM} we show that that among the models with high test loglikelihood ($>-55$), $\mathcal D_{\text{S}\rightarrow\text{T}}$, but not $\mathcal D_{\text{T}\rightarrow\text{S}}$, is highly correlated to test loglikelihood, while $\mathcal D_{\text{T}\rightarrow\text{S}}$ shows a close relationship to Average 10 shot MSE error.

\begin{figure}[h!]
    \centering
    \includegraphics[width=0.45\linewidth]{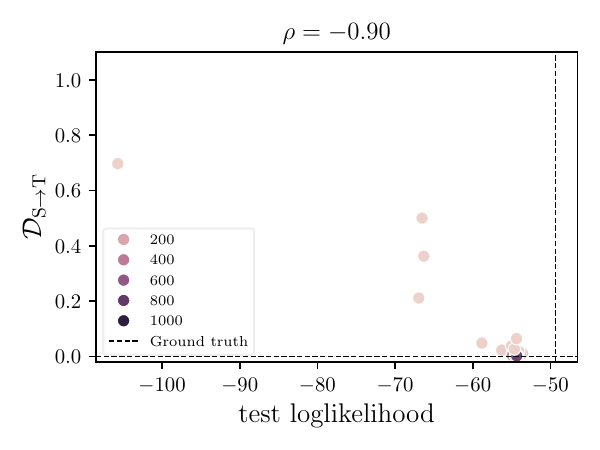}
    \includegraphics[width=0.45\linewidth]{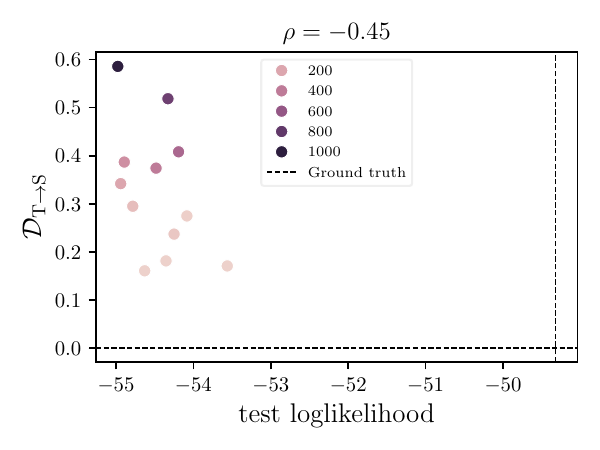}
    \includegraphics[width=0.45\linewidth]{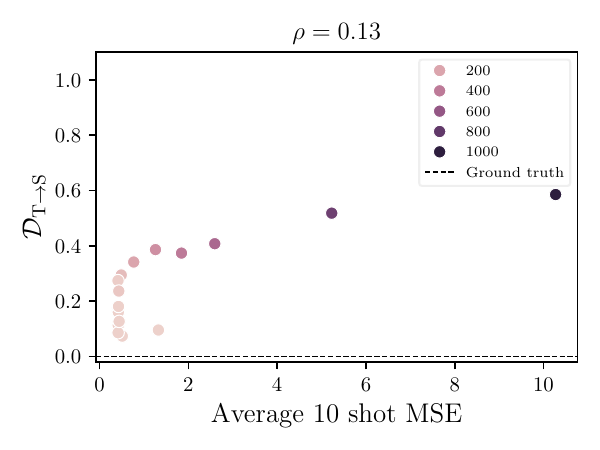}
    \caption{\textbf{Left to right}: Student-teacher results for Linear Gaussian State Space Models.  We report loglikelihood instead of co-smoothing, and $k$-shot MSE instead of $k$-shot co-smoothing.}
    \label{fig:kshot-LGSSM}
\end{figure}

%% file: supporting_information/hmm_network_visualisations.tex
{\bf HMM network visualisations}
\label{sec:HMMalledges}

In the main text Fig.~\ref{fig:hmm graphs} we visualised the teacher and two student HMMs as graphs of fractional traffic volume on states and transitions. For clarity we dropped the low probability edges with values lower than $0.01$. In Fig.~\ref{fig:HMMalledges}, we show the same models with all the edges visualised.

\begin{figure}
    \centering
    \includegraphics[width=\linewidth]{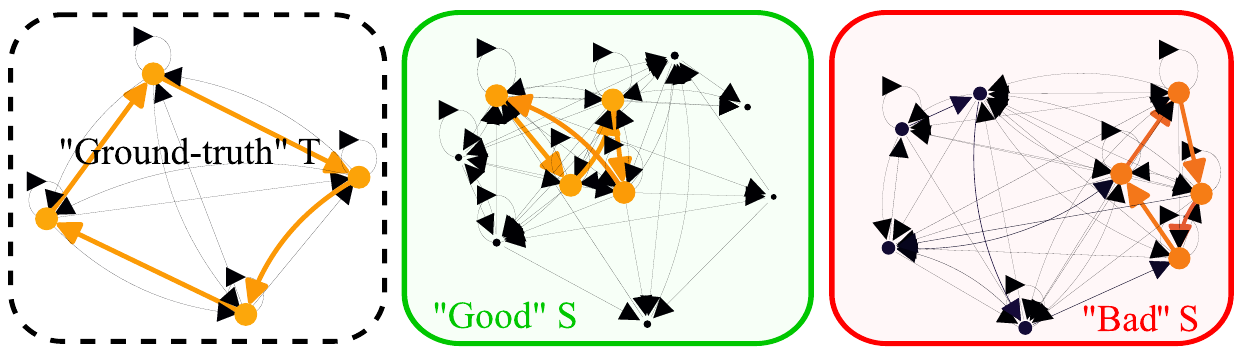}
    \caption{HMM visualisations with all transitions included.}
    \label{fig:HMMalledges}
\end{figure}

%% file: supporting_information/fewshot_not_simply_hard_cosmoothing.tex
{\bf Few-shot co-smoothing is not simply hard co-smoothing (variations of HMM student-teacher experiments).}
\label{sec: not simply hard o-smoothing}
Few-shot co-smoothing is a more difficult metric than standard co-smoothing. Thus, it might seem that any increase in the difficulty of will yield similar results. To show this is not the case, we use standard co-smoothing with fewer held-in neurons (Fig.~\ref{fig: not just harder}). The score is lower (because it's more difficult), but does not discriminate models. 

\begin{figure}[h!]
    \centering
    \includegraphics[width=0.4\linewidth]{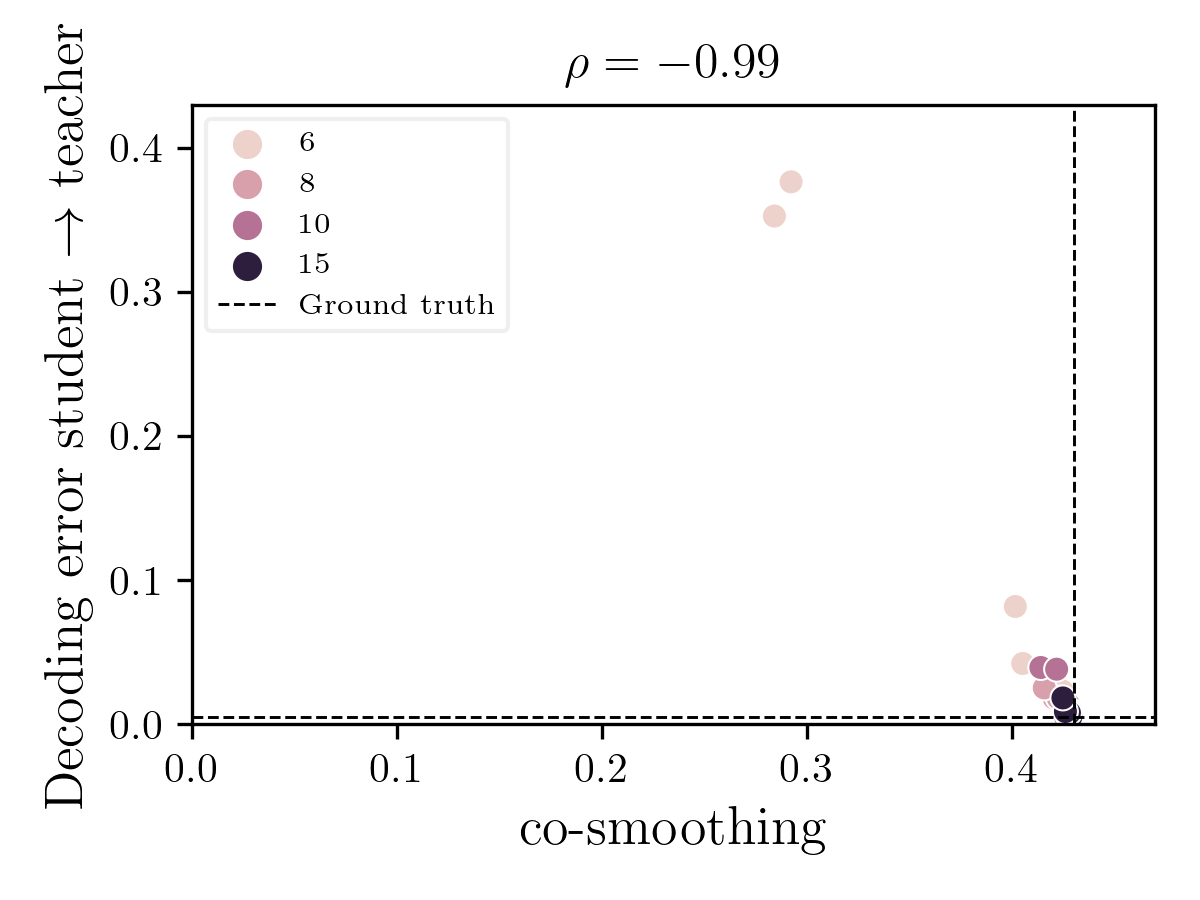}
    \includegraphics[width=0.4\linewidth]{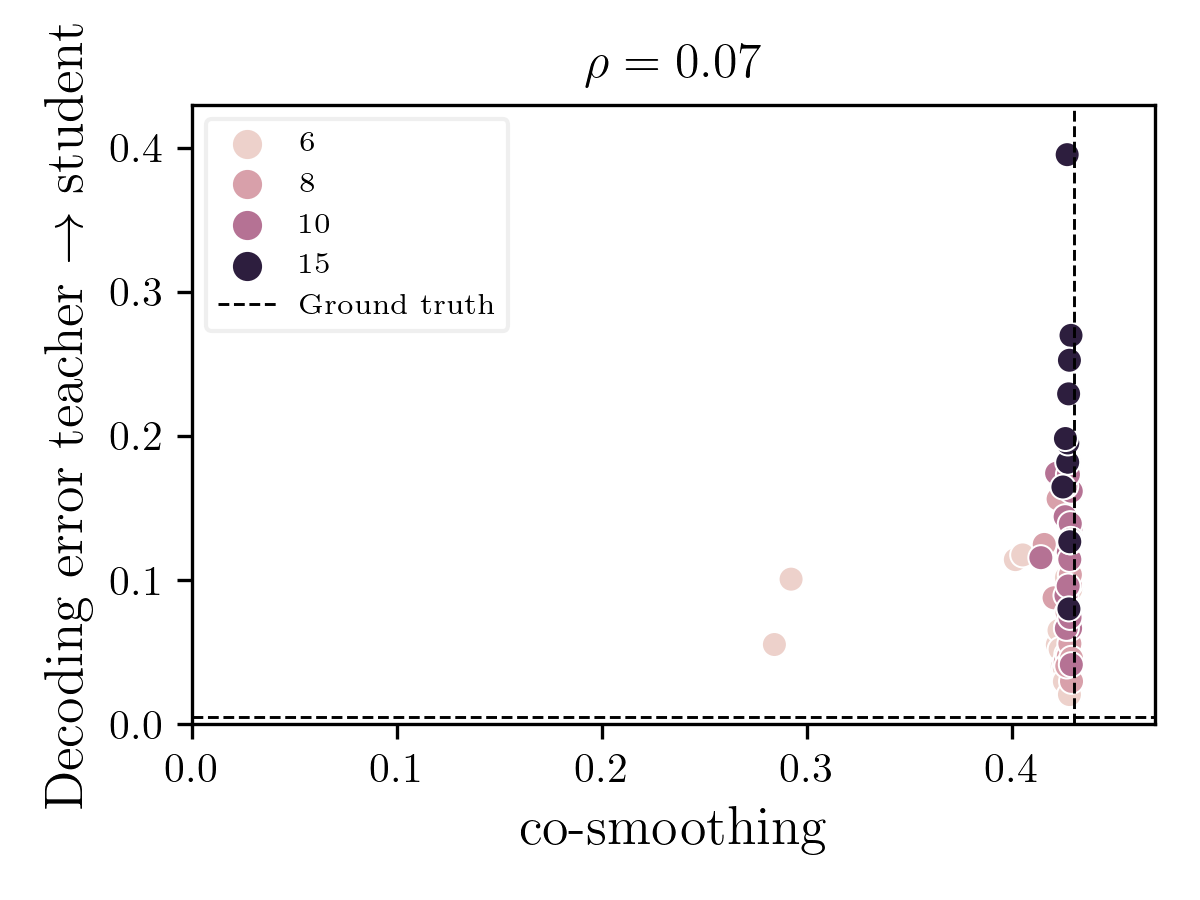}
    \includegraphics[width=0.4\linewidth]{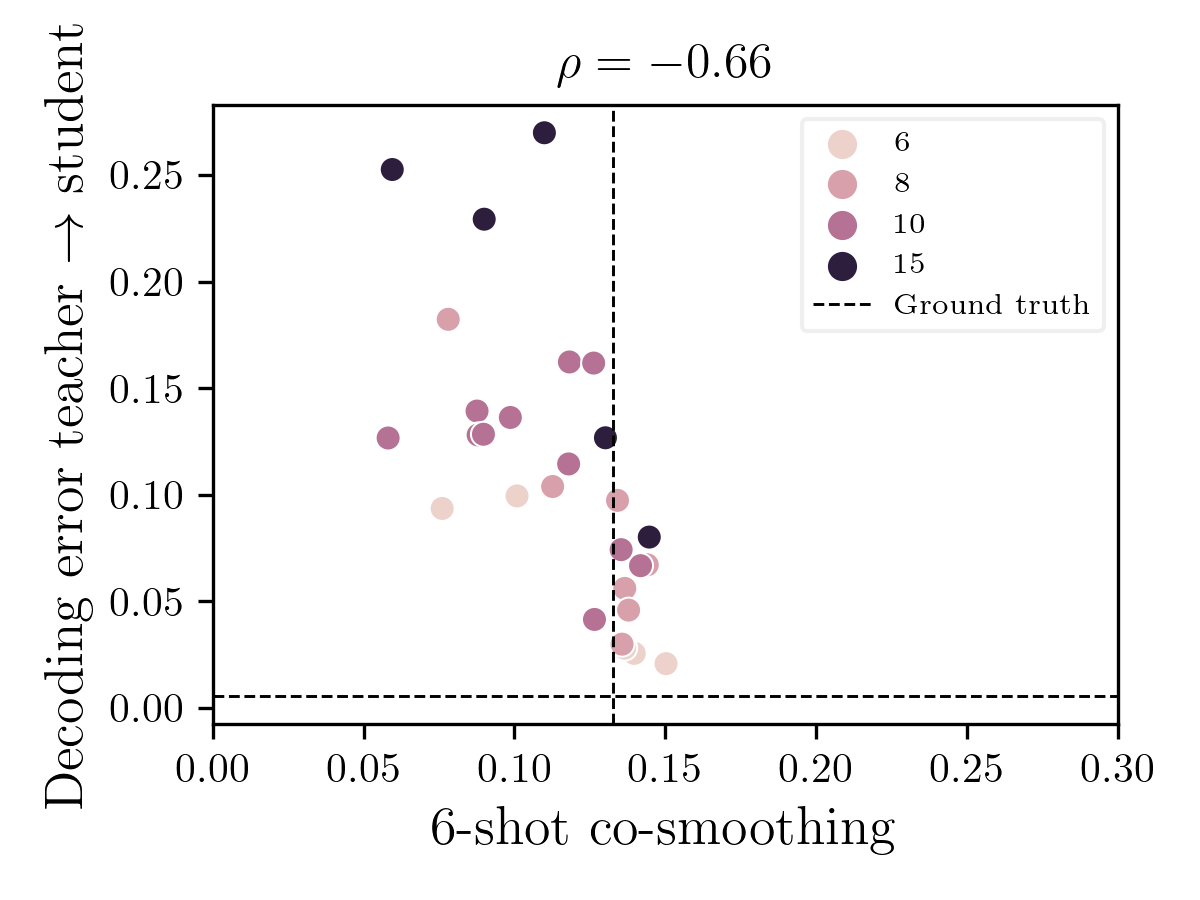}
\\    
    \includegraphics[width=0.4\linewidth]{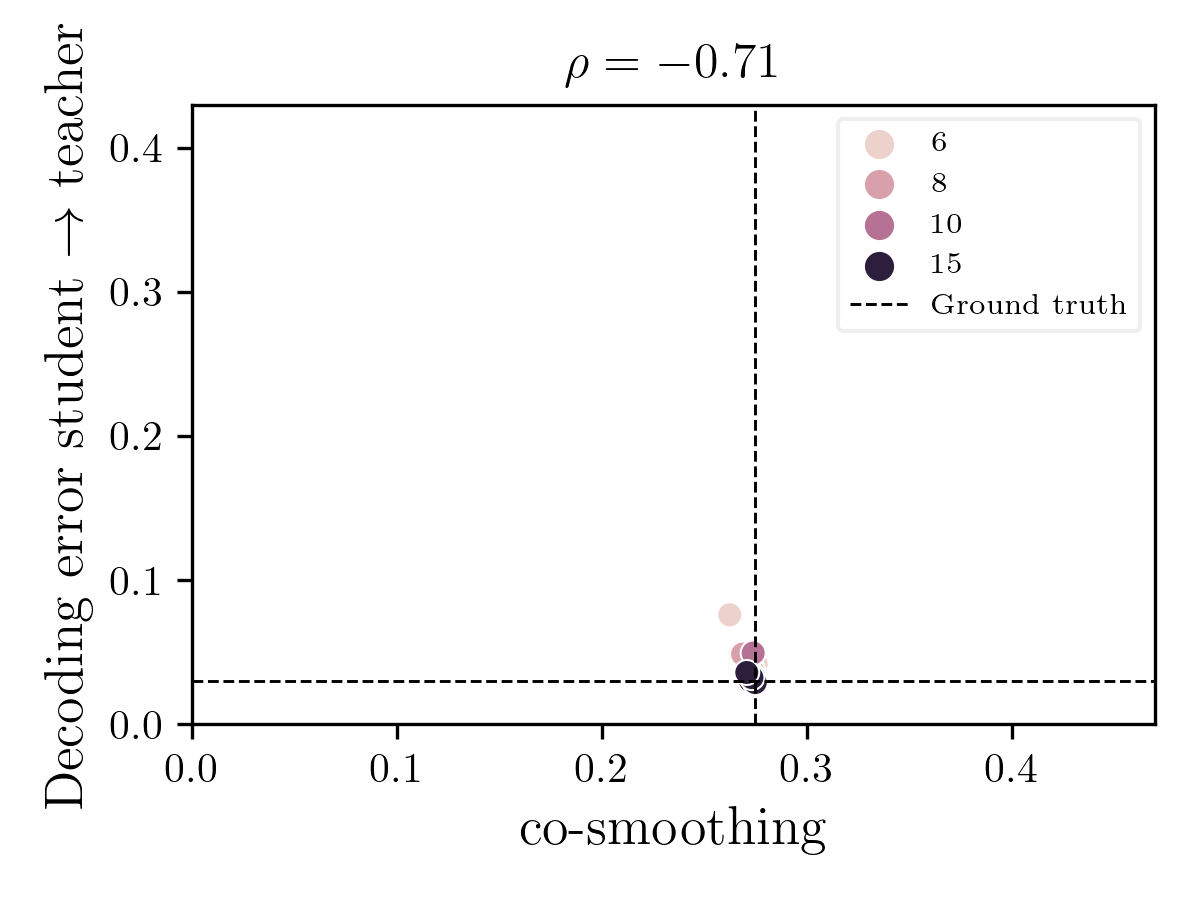}
    \includegraphics[width=0.4\linewidth]{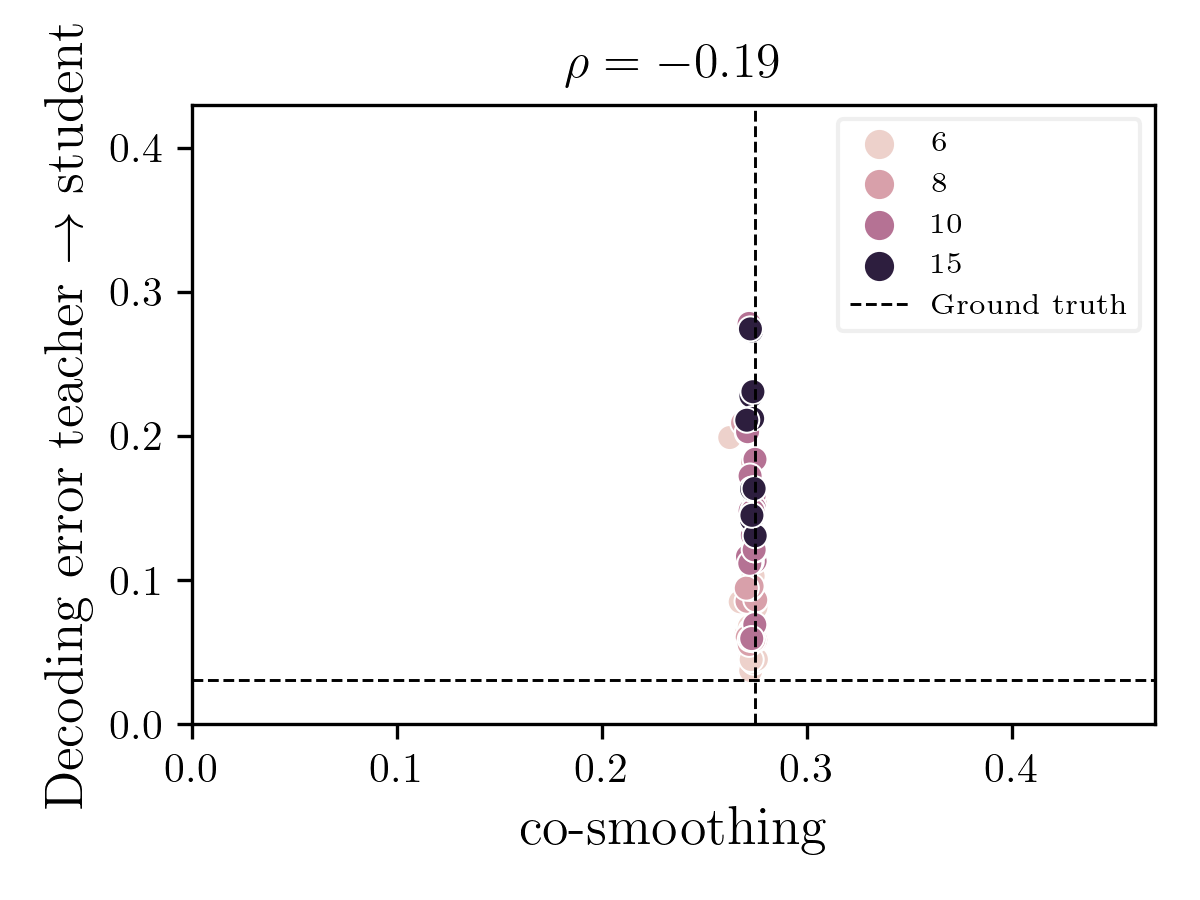}
    \includegraphics[width=0.4\linewidth]{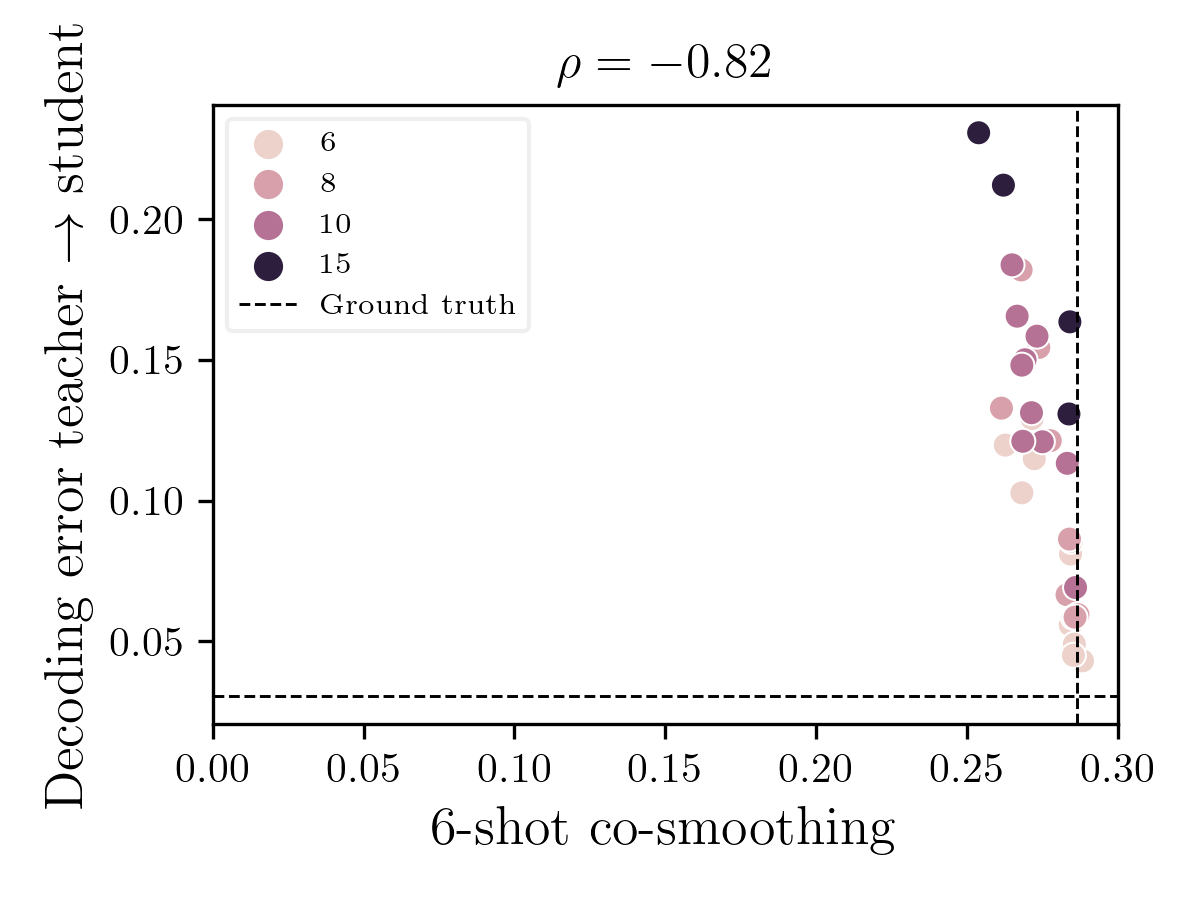}

    \caption{Making co-smoothing harder does not discriminate between models. \textbf{Top three:} Increasing the number of held out neurons from $N^{\text{out}}=50$ to $N^{\text{out}}=100$. First two panels: Same as main text Fig.~\ref{fig:hmms}CD. Lower panel: Same as main text Fig.~\ref{fig:hmm fewshot}B. \textbf{Bottom three:} Decreasing the number of held-in and held-out neurons to $N^{\text{in}}=5$, $N^{\text{out}}=5$, $N^{\text{$k$-out}}=50$. Panels as in top row. The score does decrease because the problem is harder, but co-smoothing is still not indicative of good models while few-shot is. }
    \label{fig: not just harder}
\end{figure}

%% file: supporting_information_figsandtablesremoved/code_repositories.tex
{\bf Code repositories.}
\label{sec: code}

The experiments done in this work are largely based on code repositories from previous works. The following repositories were used or developed in this work:

\begin{itemize}
\item \url{https://github.com/KabirDabholkar/ComputationThroughDynamicsBenchmark.git} - Training and analysis of RNNs and NODE SAEs \cite{versteeg_computation-through-dynamics_2025}
\item \url{https://github.com/KabirDabholkar/hmm_analysis} - Training and analysis of HMMs, implemented in \texttt{dynamax} \cite{linderman_dynamax_2025}
\item \url{https://github.com/KabirDabholkar/ssm_analysis} - Training and analysis of LGSSMs, implemented in \texttt{dynamax} \cite{linderman_dynamax_2025}
\item \url{https://github.com/KabirDabholkar/nlb_tools_fewshot} - Evaluation of SOTA models: co-smoothing, few-shot co-smoothing, cycle-consistency, and cross-decoding \citep{pei_neural_2021}
\item \url{https://github.com/KabirDabholkar/STNDT_fewshot} - Training STNDT models \citep{le2022stndt,ye_representation_2021,pei_neural_2021,nguyen_transformers_2019,huang_improving_2020}
\end{itemize}

%% file: supporting_information_figsandtablesremoved/time_cost_computing_fewshot_cosmoothing.tex
{\bf Time cost of computing few-shot co-smoothing}
\label{sec: fewshot compute time}

The compute time depends on several factors. It scales with the number of trials $k$, $T$ the number of samples per trial, the number of neurons $N^\text{$k$-out}$ and number of repeated resamples $s$ and fitting of the regressor. Each repetition and neuron is an independent regression and therefore can computed in parallel, provided the compute resources are available. For \texttt{mc\_maze\_20} each repetition for $k=64$, took $0.62 \pm 0.06$ seconds and we iterated over $s=12$ such regressions, requiring a total of $7.44$ seconds.